\documentclass{article}

\usepackage{authblk}
\usepackage{microtype}
\usepackage[utf8]{inputenc} 
\usepackage[T1]{fontenc}    
\usepackage{hyperref}       
\usepackage{url}            
\usepackage{booktabs}       
\usepackage{amsfonts}       
\usepackage{nicefrac}       
\usepackage{microtype}  
\usepackage{amsmath}
\usepackage{algorithm}
\usepackage{algorithmic}
\usepackage{times}
\usepackage{mathtools}
\usepackage[numbers]{natbib}
\usepackage{color}
\usepackage{siunitx}
\usepackage{multirow}
\usepackage{amsthm}
\usepackage{mathrsfs}
\usepackage{textcomp}
\usepackage{graphicx}
\usepackage{caption}
\usepackage{xspace}
\usepackage[export]{adjustbox}
\usepackage{float}
\usepackage{enumitem}
\usepackage{listings}
\usepackage{bbm}
\usepackage{wrapfig}
\usepackage{amsthm}
\usepackage{subcaption}
\usepackage{wrapfig}
\usepackage{xcolor}
\usepackage{tikz}

\usepackage{amsmath,amssymb,amsfonts,bm}

\newtheorem{prop}{Proposition}

\def\eqref#1{equation~\ref{#1}}
\def\1{\bm{1}}

\def\cB{\mathcal{B}}

\def\cD{\mathcal{D}}

\def\cH{\mathcal{H}}
\def\cI{\mathcal{I}}

\def\cN{\mathcal{N}}

\def\rrm{{\textnormal{m}}}
\def\rv{{\textnormal{v}}}

\def\rA{{\textnormal{A}}}

\def\rE{{\textnormal{E}}}

\def\rrM{{\textnormal{M}}}
\def\rU{{\textnormal{U}}}

\def\rX{{\textnormal{X}}}
\def\rY{{\textnormal{Y}}}
\def\rZ{{\textnormal{Z}}}

\def\rva{{\mathbf{a}}}
\def\rvb{{\mathbf{b}}}

\def\rve{{\mathbf{e}}}

\def\rvu{{\mathbf{i}}}

\def\rvm{{\mathbf{m}}}

\def\rvu{{\mathbf{u}}}
\def\rvv{{\mathbf{v}}}

\def\rvx{{\mathbf{x}}}
\def\rvy{{\mathbf{y}}}
\def\rvz{{\mathbf{z}}}
\def\rvtheta{{\bm{\theta}}}
\def\rvphi{{\bm{\phi}}}

\def\vzero{{\bm{0}}}

\def\vmu{{\bm{\mu}}}

\def\valpha{{\bm{\alpha}}}
\def\vbeta{{\bm{\beta}}}
\def\vsigma{{\bm{\sigma}}}
\def\vC{{\bm{C}}}

\DeclareMathAlphabet{\mathsfit}{\encodingdefault}{\sfdefault}{m}{sl}
\SetMathAlphabet{\mathsfit}{bold}{\encodingdefault}{\sfdefault}{bx}{n}

\def\gE{{\mathcal{E}}}

\def\gG{{\mathcal{G}}}

\def\gV{{\mathcal{V}}}

\def\sC{{\mathbb{C}}}

\def\sS{{\mathbb{S}}}

\newcommand{\pdata}{p_{\rm{data}}}
\newcommand{\E}{\mathbb{E}}
\newcommand{\I}{\mathcal{I}}
\newcommand{\Ls}{\mathcal{L}}

\newcommand{\lreg}{\Ls_{\mathrm{str\_reg}}}
\newcommand{\ldist}{\Ls_{\mathrm{dist}}}

\newcommand{\pp}{p_{\rvtheta}}
\newcommand{\qq}{q_{\rvphi}}
\newcommand{\qqt}{\tilde{q}_{\rvphi}}

\newcommand{\parents}{\mathrm{\emph{Pa}}} % See usage in notation.tex. Chosen to match Daphne's book.

\DeclarePairedDelimiterX{\infdivx}[2]{(}{)}{%
  #1\;\delimsize\|\;#2%
}
\DeclarePairedDelimiterX{\semicolondiv}[2]{(}{)}{%
  #1\;\delimsize;\;#2%
}

\newcommand{\kl}{D_{\mathrm{KL}}\infdivx}
\newcommand{\js}{D_{\mathrm{JS}}\infdivx}
\newcommand{\projkl}{\mathbb{D}\infdivx}

\newcommand{\miq}{\cI_q\semicolondiv}

\newcommand{\mynorm}[1]{\left\lVert#1\right\rVert}

\usepackage[whole]{bxcjkjatype} % Japanese lang. support

\usepackage{xcolor}
\hypersetup{
    colorlinks,
    anchorcolor={blue!50!black},
    linkcolor={blue!50!black},
    citecolor={blue!50!black},
    urlcolor={blue}
}

\usepackage{footnote}
\makesavenoteenv{tabular}

\usepackage[accepted]{icml2020}

\icmltitlerunning{Learning Structured Latent Factors from Dependent Data}

\begin{document}

\twocolumn[
\icmltitle{Learning Structured Latent Factors from Dependent Data:\\A Generative Model Framework from Information-Theoretic Perspective}

% It is OKAY to include author informatiosn, even for blind


\icmlsetsymbol{equal}{*}

\begin{icmlauthorlist}
\icmlauthor{Ruixiang Zhang}{mila,udem}
\icmlauthor{Masanori Koyama}{pfn}
\icmlauthor{Katsuhiko Ishiguro}{pfn}
\end{icmlauthorlist}

\icmlaffiliation{mila}{Mila}
\icmlaffiliation{udem}{Université de Montréal, Montréal, Canada}
\icmlaffiliation{pfn}{Preferred Networks Inc., Tokyo, Japan}

\icmlcorrespondingauthor{Ruixiang Zhang}{ruixiang.zhang@umontreal.ca}

\icmlkeywords{Machine Learning, ICML}

\vskip 0.3in
]



\printAffiliationsAndNotice{}  % leave blank if no need to mention equal contribution

\begin{abstract}
Learning controllable and generalizable representation of multivariate data with desired structural properties remains a fundamental problem in machine learning.
In this paper, we present a novel framework for learning generative models with various underlying structures in the latent space.
We represent the inductive bias in the form of mask variables to model the dependency structure in the graphical model and extend the theory of multivariate information bottleneck~\cite{mib} to enforce it.
Our model provides a principled approach to learn a set of semantically meaningful latent factors that reflect various types of desired structures like capturing correlation or encoding invariance, while also offering the flexibility to automatically estimate the dependency structure from data.
We show that our framework unifies many existing generative models and can be applied to a variety of tasks, including multi-modal data modeling, algorithmic fairness, and out-of-distribution generalization.
\end{abstract}

\section{Introduction}
\label{sec:intro}

Learning structured latent representation of multivariate data is a fundamental problem in machine learning.
Many latent variable generative models have been proposed to date based on different inductive biases that reflect the model's assumptions or people's domain knowledge.
For instance, the objectives of the family of  $\beta$-VAEs~\cite{betavae, tcvae, kimdisentangle} try to enforce a coordinate-wise independent structure among latent variables to discover disentangled factors of variations.

While these methods have been proven useful in the field of applications on which they were evaluated, most of them are built-in rather heuristic ways to encode the desired structure. One usually needs to construct an entirely different model whenever the domain of application changes.  In general, the type of inductive bias differs significantly across different applications. It is a burden to craft a different architecture for each application, and there have not been many studies done for the general and unified way of explicitly representing an inductive bias to be enforced in generative models.  

In this paper, we propose a framework of generative model that can represent various types of inductive biases in the form of Bayesian networks.
Our method can not only unify many existing generative models in previous studies, but it also can lead to new insights about establishing connections between different models across different domains and extending them to new applications.

We summarize our contributions in this work as: (i)~We propose a novel general framework of probabilistic generative model with explicit dependency structure representation to learn structured latent representation of multivariate data. (ii)~We propose an information-theoretic training objective by generalizing the multivariate information bottleneck theory to encode prior knowledge or impose inductive bias.~(Sec.~\ref{sec:framework_mib}) (iii)~We propose a flexible and tractable inference model with linear number of inference networks coupled with super-exponential number of possible dependency structures to model exponential number of inference distributions.~(Sec.~\ref{sec:framework_inference}) (iv)~We show that our proposed framework unifies many existing models and demonstrate its effectiveness in different application tasks,  including multi-modal data generative modeling, algorithmic fairness, and out-of-distribution generalization.

\section{Background}
\subsection{Notations}
\label{sec:bg_notation}
We use capital letters (i.e. $\rX \equiv \rX_{1:N}$ ) to denote a vector of $N$ random variables, and lower case letters (i.e. $\rvx$) for the values. We use $P(\rX)$ to denote the probability distribution and corresponding density with $p(\rvx)$. Given a set $\sS \subseteq \left\{ 1, 2, \ldots, N\right\}$ of indexes, we use $\rX^\sS \equiv  \left[\rX_i \right]_{i \in \sS}$ to represent corresponding subset of random variables. Similar notation is used for binary indicator vector $\rvb$ that $\rX^\rvb \equiv  \left[\rX_i \right]_{\rvb_i =1}$.

\subsection{Probability and information theory}
A \textbf{Bayesian network} $\gG \equiv \left( \gV, \gE\right)$ defined over random variables $\rX$ is a directed acyclic graph, consisting of a set of nodes $\gV \equiv \left\{ \rX_i \right\}_{i=1}^N$ and a set of directed edges $\gE \subseteq \gV^2$.
A node $\rvu$ is called a \emph{parent} of $\rvv$ if $(\rvv, \rvu) \in \gE$, and for each random variable $\rX_i$, the set of parents of $\rX_i$ is denoted by $\parents^{\gG}_{\rX_i}$.
We use $\gG^{\emptyset}$ to denote an empty Bayesian network $\gG^{\emptyset} \equiv (\gV, \emptyset)$.
If a distribution $P(\rX)$ is consistent with a Bayesian network $\gG$, then it can be factorized as $p(\rvx)=\prod_i p(\rvx_i \mid \parents^{\gG}_{\rvx_i})$, denoted by $p \models \gG$.

We then briefly introduce the information theory concepts used in this paper here. 
The Shannon \textbf{Entropy} is defined as $\cH (\rX) =  - \E_{p(\rvx)}\log p(\rvx)$ to measure the average number of bits needed to encode values of $\rX \sim P(\rX)$.
The \textbf{Kullback–Leibler Divergence}~(KLD) is one of the most fundamental distance between probability distributions defined as $\kl{P}{Q} = \E_{p}\log \frac{p}{q}$.
\textbf{Mutual Information} $\cI (\rX; \rY) = \E_{p(\rvx, \rvy)} \log \frac{p(\rvx, \rvy)}{p(\rvx)p(\rvy)}$ quantifies the mutual dependence between two random variables $\rX$ and $\rY$. The mutual information is zero if and only if $\rX$ and $\rY$ are independent.
\textbf{Multi-Information} is one of multivariate mutual information  defined as $\cI (\rX_{1}, \dots, \rX_{N}) = \kl{p(\rvx_{1:N})}{\prod_{i=1}^N p(\rvx_i)}$, which generalizes the mutual information concept to quantify the multivariate statistical independence for arbitrary number of random variables.
\cite{jsd} proposed a \textbf{generalized Jensen-Shannon divergence} defined as $D^{\bm{\pi}}_{\mathrm{JS}} = \cH \left(\sum_{i=1}^N \pi_i P_i \right) - \sum_{i=1}^N \pi_i \cH(P_i)$, where $P_1, \ldots, P_N$ are $N$ distributions with weights $\pi_1, \ldots, \pi_N$. Commonly used Jensen-Shannon divergence~(JSD) can be seen as a special case when $N = 2$ and $\pi_1=\pi_2 = \frac{1}{2}$.
\cite{jsd_abs_mean} further generalized the arithmetic mean $\sum_{i=1}^N \pi_i P_i$ to other abstract means and proposed closed-form results of geometric mean of exponential family distributions and the divergence among them.

As shown in~\cite{mib}, if a distribution $P(\rX_{1:N})$ is consistent with a Bayesian network $\gG$, the multi-information $\cI(\rX)$ can be expressed as a sum of all local mutual information terms: $\cI(\rX) = \sum_{i=1}^N \cI \left( \rX_i; \parents^{\gG}_{\rX_i} \right)$.
Then the multi-information in $P(\rX)$ with respect to an arbitrary valid Bayesian network $\gG$ can be defined \footnote{Note that $P(\rX)$ is not necessarily consistent with $\gG$ here} as $\cI^{\gG}_{p}(\rX) = \sum_{i=1}^N \cI^{\gG}_p \left( \rX_i; \parents^{\gG}_{\rX_i} \right)$.
The \textbf{M-projection}~\cite{pgm_book,mib} of distribution $P(\rX)$ to the set of distribution that is consistent with a Bayesian network $\gG$ is defined as $\projkl{p}{\gG} = \min_{q \models \gG} \kl{p}{q}$. Then the following results was introduced in~\cite{mib}
\begin{equation}
\projkl{p}{\gG} = \min_{q \models \gG} \kl{p}{q} = \cI_{p}(\rX) - \cI^{\gG}_{p}(\rX)
\end{equation}
where we use subscript to denote the distribution that the mutual information term is evaluated with, and we use superscript to denote the graphical structure that the indicates set of parent nodes used in $\cI^{\gG}_p \left( \rX_i; \parents^{\gG}_{\rX_i} \right)$.

\subsection{Variational autoencoder}
Variational autoencoder~(VAE)~\cite{kingma-vae} is a probabilistic latent variable generative model $\pp(\rvx, \rvz) = \pp(\rvz)\pp(\rvx \mid \rvz)$, where $\pp(\rvz)$ is the prior of latent variables $\rZ$ and $\pp(\rvx \mid \rvz)$ is the likelihood distribution for observed variable $\rX$.
The generative model is often optimized together with a tractable distribution $\qq(\rvz \mid \rvx)$ that approximates the posterior distribution.
%An inference model is proposed for tractable approximate posterior inference $\qq(\rvz \mid \rvx)$.
The distributions are usually parametrized by neural networks with parameters $\rvtheta$ and $\rvphi$.
The inference model and generation model are jointly optimized by a lower-bound of the KLD between $\qq$ and $\pp$ in the augmented space $(\rX, \rZ)$, namely \emph{ELBO}:
\begin{equation}
\E_{\qq} \log \pp(\rvx | \rvz) - \E_{\qq(\rvx)} \kl{\qq(\rvz \mid \rvx)}{\pp(\rvz)} \equiv \Ls_{\mathrm{ELBO}}\\
\end{equation}
Note $- \Ls_{\mathrm{ELBO}} \ge \kl{\qq(\rvx)\qq(\rvz \mid \rvx)}{\pp(\rvz) \pp(\rvx \mid \rvz)}$ where $\qq(\rvx) = \pdata(\rvx)$ denotes the empirical data distribution.
The above objective can be optimized efficiently with the re-parametrization trick~\cite{kingma-vae,kingma2019an}.
\subsection{Multivariate information bottleneck}
\label{sec:bg_mib}
Multivariate Information Bottleneck~(MIB) theory proposed by~\cite{mib,mib_slonim} extends the information bottleneck theory~\cite{ib_1999} to multivariate setting.
Given a set of observed variable $\rX$, MIB framework introduced a Bayesian network $\gG^{\mathrm{in}}$ to define the solution space of latent variables $\rZ$ as $q(\rX, \rZ) \models \gG^{\mathrm{in}}$.
Another Bayesian network $\gG^{\mathrm{out}}$ is introduced to specify the relevant information to be preserved in $\rZ$.
Then the MIB functional objective is defined as $\Ls_{MIB}^1(q) = \cI_{q}^{\gG^{\mathrm{in}}}(\rX) - \beta \cI^{\gG^{\mathrm{out}}}_{q}(\rX)$.
An alternative structural MIB functional objective is defined as $\Ls_{MIB}^2(q) = \cI_{q}^{\gG^{\mathrm{in}}}(\rX) + \gamma \projkl{q(\rvx, \rvz)}{\gG^{\mathrm{out}}}$, and further relaxed by~\cite{mib_hidden} as $\Ls_{MIB}^2 (q,p) = \cI_{q}^{\gG^{\mathrm{in}}}(\rX) + \gamma \kl{q(\rvx, \rvz)}{p(\rvx, \rvz)}$. We refer to~\cite{mib,mib_hidden} for more details of MIB theory.

\label{sec:mib}

\section{Framework}
\label{sec:framework}
\subsection{Preliminaries}
Given a dataset $\cD = \left\{\rvx^{d}\right\}_{d=1}^{|\cD|}$, we assume that observations are generated from some random process governed by a set of latent factors, which could be categorized into two types: private latent factors $\rU \equiv \rU_{1: N} \equiv \left[\rU_{1}, \rU_{2}, \ldots, \rU_{N}\right]$ and common latent factors $\rZ \equiv \rZ_{1: M} \equiv\left[\rZ_{1}, \rZ_{2}, \ldots, \rZ_{M}\right]$. We use $\rU_i$ to denote the latent factors that are exclusive to the variable $\rvx_i$ and assume a jointly independent prior distribution  $P(\rU)$. We use $\rZ$ to denote the latent factors that are possibly shared by some subset of observed variables and assume a prior distribution $P(\rZ)$. The dimension of each $\rU_i$ and $\rZ_j$ is arbitrary.

% We explicitly model the connectivity between $\rX$ and $\rZ$ using a binary matrix $\rrM \equiv (\rrM_{ij}) \in \left\{0, 1\right\}_{N \times M}$, $\rrm_{ij} = 1$ when the latent factor $\rZ_j$ contributes to the random generation process of $\rX_i$, or otherwise $\rrM_{ij} = 0$. Let $\rrM^x_i = (\rrM_{i1}, \ldots, \rrM_{iM})$ denote the $i$-th row of $\rrM$ and $\rrM^z_j = (\rrM_{1j}, \ldots, \rrM_{Nj})$ denote the $j$-th column.

\subsection{Generative model with explicit dependency structure representation}
\label{sec:generative_model}
%structure variable $\rrm_{ij} \sim \mathrm{Bernnouli}(\bm{\theta}^m_{ij})$ and $p(\rvm) = \prod_{i=1}^N\prod_{j=1}^M p(\rvm_{ij}) $
\textbf{Generation model}
We explicitly model the dependency structure from $\rZ$ to $\rX$ in the random generation process with a binary matrix variable $\rrM^p \equiv \left[\rrM^p_{ij}\right] \in \left\{0, 1\right\}_{N \times M}$.
$\rrM^p_{ij} = 1$ when the latent factor $\rZ_j$ contributes to the random generation process of $\rX_i$, or otherwise $\rrM^p_{ij} = 0$.
Let $\rrM^p_i = \left[\rrM^p_{i1}, \rrM^p_{i2}, \ldots, \rrM^p_{iM}\right]$ denote the $i$-th row of $\rrM$. 
we can define our generative model $p_\rvtheta(\rvx, \rvz, \rvu)$ as
\begin{equation}
\label{eq:p_gen}
    p_\rvtheta(\rvx, \rvz, \rvu) = \pp(\rvz) \prod_{i=1}^N \pp(\rvu_i) \prod_{i=1}^N \pp(\rvx_i \mid \rvz^{\rvm^p_i}, \rvu_i)
\end{equation}

where $\rvtheta$ is the parameter for parameterizing the generation model distribution.
The structure of the generation model is illustrated by Bayesian network $\gG^p_{\mathrm{full}}$ in Figure~\ref{fig:bn_mvae}, where the structural variable $\rrM^p$ is depicted as the dashed arrows.

\textbf{Inference}
We introduce an inference model to approximate the true posterior distributions.
We introduce another binary matrix variable $\rrM^q \equiv \left[\rrM^q_{ij}\right] \in \left\{0, 1\right\}_{N \times M}$.
$\rrM^q_{ij} = 1$ when the observed variable $\rX_i$ contributes to the inference process of $\rZ_j$, or otherwise $\rrM^q_{ij} = 0$.
Let $\rrM^q_j = \left[\rrM^q_{1j}, \rrM^q_{2j}, \ldots, \rrM^q_{Nj}\right]$ denote the $j$-th column of $\rrM^q$.
We assume that latent variables are conditional jointly independent given observed variables.
Then we could define our inference model $q_\rvphi(\rvx, \rvz, \rvu)$ as:
\begin{equation}
\label{eq:q_inf}
q_\rvphi(\rvx, \rvz, \rvu) = \qq(\rvx) \prod_{i=1}^N \qq(\rvu_i \mid \rvx_i) \prod_{j=1}^M \qq(\rvz_j \mid \rvx^{\rvm^q_j}) 
\end{equation}
where $\rvphi$ is the parameter for parameterizing the inference distribution.
The structure of the inference model is illustrated by Bayesian network $\gG^q_{\mathrm{full}}$ in Figure~\ref{fig:bn_mvae}, where the structural variable $\rrM^q$ is depicted as the dashed arrows.

\subsection{Learning from information-theoretic perspective}
\label{sec:framework_mib}
We motivate our learning objective based on the MIB~\cite{mib} theory.
We can define a Bayesian network $\gG^q \equiv \left( \gV^q, \gE^q \right)$ that is consistent with the inference model distribution $q_\rvphi(\rvx, \rvz, \rvu) \models \gG^q$ according to $\rrM^q$.
A directed edge from $\rX_i$ to $\rU_i$ is added for each $i \in \left\{1, 2, \ldots, N\right\}$ and an edge from $\rX_i$ to $\rZ_j$ is added if and only if $\rrm^q_{ij} = 1$.
Note that we could omit all edges between observed variables in $\gG^q$ as shown in~\cite{mib, mib_hidden}.
A Bayesian network $\gG^p \equiv \left( \gV^p, \gE^p \right)$ can be constructed according to $\rrM^p$ in a similar way.
As introduced in Eq. \ref{sec:bg_mib}, we have the following structural variational objective from the MIB theory:
\begin{equation}
\label{eq:mib_loss}
\min_{p_\rvtheta \models \gG^p, q_\rvphi \models \gG^q} \Ls(\rvtheta, \rvphi) = \cI^{\gG^q}_q + \gamma \kl{q_\rvphi}{p_\rvtheta}
\end{equation}
The above objective provides a principled way to trade-off between (i)the compactness of learned latent representation measured by $\cI^{\gG^q}_q$, and (ii)the consistency between $\qq(\rvx, \rvz, \rvu)$ and $\pp(\rvx, \rvz, \rvu)$ measured by the KLD, through $\gamma > 0$. 

We further generalize this objective to enable encoding a broader class of prior knowledge or desired structures into the latent space.
We prescribe the dependency structure and conditional independence rules that the learned joint distribution of $\left( \rvx, \rvz, \rvu \right)$ should follow, in the form of a set of Bayesian networks $\left\{ \gG^k \equiv \left(\gV^k, \gE^k \right) \right\}, k = 1,\ldots, K$.
We optimize over the inference distributions $q_\rvphi$ to make it as consistent with $\gG^k$ as possible, measured by its M-projection to $\gG^k$.
Formally we have the following constrained optimization objective:
\begin{equation}
\begin{aligned}
&\min_{p_\rvtheta \models \gG^p, q_\rvphi \models \gG^q} \Ls(\rvtheta, \rvphi) = \kl{\qq}{\pp}\\
&\mathrm{s.t.}\quad\projkl{q_\rvphi}{\gG^k} = 0 \quad k = 1, 2, \ldots, K
\end{aligned}
\end{equation}
In this way, we impose the preferences over the structure of learned distributions as explicit constraints. We relax the above constrained optimization objective with generalized Lagrangian
\begin{equation}
\max_{\vbeta \ge \vzero} \min_{p_\rvtheta \models \gG^p, q_\rvphi \models \gG^q} \Ls = \kl{q_\rvphi}{p_\rvtheta} + \sum_{k=1}^K \beta_k \projkl{q_\rvphi}{\gG^k}
\end{equation}
where $\vbeta \equiv \left[ \beta_1, \beta_2, \ldots, \beta_K \right]$ is the vector of Lagrangian multipliers.
In this work we fix $\vbeta$ as constant hyper-parameters, governing the trade-off between structural regularization and distribution consistency matching. Following the idea proposed in~\cite{lagvae}, we could also generalize the distribution matching loss by using a vector of $T$ \emph{cost functions} $\vC \equiv \left[\sC_1, \sC_2, \ldots, \sC_T \right]$ and a vector of Lagrangian multipliers $\valpha \equiv \left[\alpha_1, \alpha_2, \ldots, \alpha_T \right]$.
Each $\sC_i$ can be any probability distribution divergence between $\qq$ and $\pp$, or any measurable cost function defined over corresponding samples.
Thus we could decompose the overall objective as
\begin{equation}
\label{eq:loss}
\begin{split}
&\Ls = \ldist + \lreg \\
&\ldist = \sum_{t=1}^T \alpha_t \sC_t(q_\rvphi \;\|\; p_\rvtheta), \quad \valpha \ge \vzero\\
&\lreg = \sum_{k=1}^K \beta_k \projkl{q_\rvphi}{\gG^k}, \quad \vbeta \ge \vzero \, . \\
\end{split}
\end{equation}
By setting $\sC_1 = \kl{\qq}{\pp}$ and $\gG^1 = \gG^{\emptyset}$, we can obtain that the original MIB structural variational objective in Eq.~\ref{eq:mib_loss} as a special case. We include the detailed proof in Appendix.~\ref{ap:framework}.

\subsection{Tractable inference and generation}
Though we have our generation and inference model defined in Sec.~\ref{sec:generative_model}, it's not clear yet how we practically parametrize $\qq$ and $\pp$ in a tractable and flexible way, to handle super-exponential number of possible structures $\rrM^p, \rrM^q$ and efficient inference and optimization.

\textbf{Inference model}~
\label{sec:framework_inference}
We identify the key desiderata of our inference model defined in~Eq.~\ref{eq:q_inf} as
(i)being compatible with any valid structure variable $\rrM^q$ and
(ii)being able to handle missing observed variables in $q(\rvz_j \mid \rvx^{\rvm^q_j})$, in an unified and principled way.
% The reason we want (ii) is that we may need to enforce the consistency of conditional distributions or marginal distributions in $\ldist$, for example in the context of multimodal data modeling or data imputation applications.
Building upon the assumption of our generation model distribution $\pp$ in~Eq.~\ref{eq:p_gen} that all observed variables $\rX$ are conditional jointly independent given $\rZ$, we have following factorized formulation in the true posterior distribution $\pp(\rvz \mid \rvx)$ by applying Bayes' rule:
\begin{equation}
\begin{aligned}
&p_\rvtheta\left(\rvz \mid \rvx^{\sS} \right)  =\frac{\pp(\rvx^{\sS} \mid \rvz) \pp(\rvz)}{\pp(\rvx^{\sS})} = \frac{\pp(\rvz)}{\pp(\rvx^{\sS})} \prod_{i \in \sS} \pp\left(\rvx_{i} \mid \rvz \right) \\
&= \frac{\pp(\rvz)}{\pp(\rvx^{\sS})} \prod_{i \in \sS} \frac{\pp(\rvz \mid \rvx_i) \pp(\rvx_i)}{\pp(\rvz)} \propto \pp(\rvz) \prod_{i \in \sS} \frac{\pp(\rvz \mid \rvx_i)}{\pp(\rvz)}
\end{aligned}
\end{equation}
where $\sS \subseteq \left\{ 1, 2, \ldots, N \right\}$.
In this way, we established the relationship between the joint posterior distribution $\pp(\rvz \mid \rvx)$ and the individual posterior distribution $\pp(\rvz \mid \rvx_i)$.
We adopt the same formulation in our inference model distribution as $\qq(\rvz \mid \rvx^{\sS}) \propto \pp(\rvz) \prod_{i \in \sS} \frac{\qq(\rvz \mid \rvx_i)}{\pp(\rvz)}$, using $N$ individual approximate posterior distributions $\qq(\rvz \mid \rvx_i)$.
In this work, we assume that $\pp(\rvz)$ and $\qq(\rvz \mid \rvx_i)$ are all following factorized Gaussian distributions.
And each individual posterior $\qq(\rvz \mid \rvx_i)$ can be represented as:
\begin{equation}
\qq(\rvz \mid \rvx_i) = \prod_{j=1}^M \qq(\rvz_j \mid \rvx_i)^{\rvm^q_{ij}}\pp(\rvz_j)^{1 - \rvm^q_{ij}}
\end{equation}
where each $\qq(\rvz_j \mid \rvx_i)$ is a multiplicative mixture between the approximated posterior $\qq(\rvz_j \mid \rvx_i)$ and the prior $\pp(\rvz_j)$, weighted by $\rvm^q_{ij}$.
Since the quotient of two Gaussian distributions is also a Gaussian under well-defined conditions, we could parametrize the quotient $\frac{\qq(\rvz \mid \rvx_i)}{\pp(\rvz)}$ using a Gaussian distribution parametrized by $\qqt(\rvz \mid \rvx_i)$. In this case
\begin{equation}
\begin{split}
&\frac{\qq(\rvz \mid \rvx_i)}{\pp(\rvz)} = \prod_{j=1}^M \frac{\qq(\rvz_j \mid \rvx_i)^{\rvm^q_{ij}}\pp(\rvz_j)^{1 - \rvm^q_{ij}}}{\pp(\rvz_j)^{\rvm^q_{ij} + 1 - \rvm^q_{ij}}}\\
&= \prod_{j=1}^M \left(\frac{\qq(\rvz_j \mid \rvx_i)}{\pp(\rvz_j)}\right)^{\rvm^q_{ij}} = \prod_{j=1}^M \left(\qqt(\rvz_j \mid \rvx_i) \right)^{\rvm^q_{ij}}
\end{split}
\end{equation}
where we use a inference network $\qqt(\rvz_j \mid \rvx_i)$ to parametrize $\qq(\rvz_j \mid \rvx_i)$ as $\qq(\rvz_j \mid \rvx_i) = \qqt(\rvz_j \mid \rvx_i)\pp(\rvz_j)$.
We show our full inference distribution $\qq(\rvz \mid \rvx)$ as:
\begin{equation}
\begin{aligned}
\qq(\rvz \mid \rvx^{\sS})
% &\propto \pp(\rvz) \prod_{i\in {\sS}} \frac{\qq(\rvz \mid \rvx_i)}{\pp(\rvz)} \\
% &= \pp(\rvz) \prod_{i \in \sS}\prod_{j=1}^M \left(\qqt(\rvz_j \mid \rvx_i) \right)^{\rvm^q_{ij}}\\
\propto \prod_{j=1}^M \left(\pp(\rvz_j)\prod_{i\in \sS} \left(\qqt(\rvz_j \mid \rvx_i) \right)^{\rvm^q_{ij}} \right)
\end{aligned}
\label{eq:inference_poe}
\end{equation}
which is a weighted product-of-experts~\cite{hinton_poe} distribution for each latent variable $\rZ_j$.
We include the detailed derivation in Appendix.~\ref{ap:framework}.
The structure variable $\rrM^q_{ij}$ controls the weight of each multiplicative component $\qqt(\rvz_j \mid \rvx_i)$ in the process of shaping the joint posterior distribution $\qq(\rvz \mid \rvx)$.
As a result of the Gaussian assumptions, the weighted product-of-experts distribution above has a closed-form solution.
Suppose $\pp(\rvz) \sim \cN\left(\vmu_0, \mathrm{diag}\left(\vsigma_0\right)\right)$, $\qqt(\rvz \mid \rvx_i) \sim \cN\left(\vmu_i, \mathrm{diag}\left(\vsigma_i\right)\right)$ for $i = 1, 2, \ldots, N$.
We introduce "dummy" variables in $\rvm^q$ that $\rvm^q_{0j} = 1$ for all $j$.
% Given any subset of observed variables $\rvx^\sS \equiv \left[\rvx_i \right], \; i \in \sS$, $\sS \subseteq \left\{ 1, 2, \ldots, N\right\}$
Then we have
\begin{equation}
\begin{split}
&\qq(\rvz \mid \rvx^\sS) \sim \cN \left( \vmu^q, \mathrm{diag}\left(\vsigma^q\right) \right) \\
& \frac{1}{\vsigma_j^q} = \sum_{i \in \sS \cup \left\{0\right\} } \frac{\rvm^q_{ij}}{\vsigma_{ij}} \quad \vmu^q_j = \frac{1}{\vsigma_j^q} \sum_{i \in \sS \cup \left\{0\right\} } \frac{\rvm^q_{ij}}{\vsigma_{ij}}\vmu_{ij} \, . \\ 
\end{split}
\end{equation}
With the derived inference model above, we are now able to model $2^N$ posterior inference distributions $\qq(\rvz \mid \rvx^\sS) \; \forall \sS$, coupled with $2^{N \times M}$ possible discrete structures $\rrM^q$, with $N$ inference networks $\qqt(\rvz \mid \rvx_i)$.
Note that the introduced distribution $\qq(\rvz \mid \rvx)$ remains valid when we extend the value of structure variable $\rrM^q$ to continuous domain $\mathbb{R}^{N \times M}$, which paves the way to gradient-based structure learning.

\textbf{Generation model}
We could parametrize our generation model $\pp$ in a symmetric way using the weighted product-of-expert distributions using $\pp(\rvx_i \mid \rvz_j)$ and $\rrM^p$.
In this work we adopt an alternative approach, due to the consideration that the Gaussian distribution assumption is inappropriate in complex raw data domain, like image pixels.
We instead use $\rrM^p$ as a gating variable and parametrize $\pp(\rvx_i \mid \rvz^{\rvm^p_i})$ in the form of $\pp(\rvx_i \mid \rvz^{\rvm^p_i}) = \pp(\rvx_i \mid \rvz \odot \rvm^p_i)$, where $\odot$ denotes element-wise multiplication. We can see that it's still tractable since the prior $\pp(\rvz)$ is known.

\subsection{Tractable optimization}
\textbf{Structural regularization $\lreg$}
Let's take a close look at the structural regularization term $\lreg$ in our training objective Eq.~\ref{eq:loss}.
As introduced in Sec.\ref{sec:bg_mib}, we have $\projkl{\qq}{\gG^k} = \sum_{\rv \in \left\{\rvx,\rvz \right\} } \miq{\rv}{\parents^{\gG^q}_{\rv}} - \sum_{\rv \in \left\{\rvx,\rvz \right\}} \miq{\rv}{\parents^{\gG^k}_{\rv}}$. 
This objective poses new challenge to estimate and optimize mutual information.
Note that any differentiable mutual information estimations and optimization methods can be applied here.
In this paper, we propose to use tractable variational lower/upper-bounds of the intractable mutual information by re-using distributions $\qq$ and $\pp$.
We refer to~\cite{mi_bounds} for a detailed review and discussion of state-of-the-art tractable mutual information optimization methods.

\begin{algorithm}[tb]
   \caption{Training with optional structure learning}
   \label{alg:learning}
\begin{algorithmic}
   \REQUIRE dataset $\cD = \left\{\rvx^{d}\right\}_{d=1}^{|\cD|}$
   \REQUIRE parameters $\rvphi, \rvtheta, \bm{\rho^q}, \bm{\rho^p}$
   \REQUIRE Bayesian Networks $\left\{ \gG^k \equiv \left(\gV^k, \gE^k \right) \right\}$
   \REQUIRE hyper-parameters $\valpha$, $\vbeta$
   \REQUIRE number of iterations to update distribution parameters  $steps\_{dist} > 0$
   \REQUIRE number of iterations to update structure parameters  $steps\_{str} \ge 0$
   \REQUIRE mini-batch size $bs$
   \REQUIRE gradient-based optimizer $opt$
   \STATE initialize all parameters $\rvphi, \rvtheta, \bm{\rho^q}, \bm{\rho^p}$
   \REPEAT
       \FOR{$step=1$ {\bfseries to} $steps\_{dist}$}
       \STATE randomly sample a mini-batch $\cB$ of size $bs$ from dataset $\cD$
       \STATE evaluate loss $\ldist^\cB$ using Eq.~\ref{eq:loss}
       \STATE compute gradients $\nabla_{\rvphi} \ldist^\cB$, $\nabla_{\rvtheta}\ldist^\cB$ 
       \STATE $opt.optimize(\left[\rvphi, \rvtheta \right], \left[ \nabla_{\rvphi} \ldist^\cB, \nabla_{\rvtheta}\ldist^\cB\right])$
       \ENDFOR
       
       \FOR{$step=1$ {\bfseries to} $steps\_{str}$}
       \STATE randomly sample a mini-batch $\cB$ of size $bs$ from dataset $\cD$
       \STATE evaluate loss $\Ls_{\mathrm{score}}^\cB$ using Eq.~\ref{eq:loss_score}
       \STATE compute gradients $\nabla_{\bm{\rho^q}} \Ls_{\mathrm{score}}^\cB$, $\nabla_{\bm{\rho^p}} \Ls_{\mathrm{score}}^\cB$ 
       \STATE $opt.optimize(\left[\bm{\rho^q}, \bm{\rho^p} \right], \left[\nabla_{\bm{\rho^q}} \Ls_{\mathrm{score}}^\cB, \nabla_{\bm{\rho^p}} \Ls_{\mathrm{score}}^\cB \right])$
       \ENDFOR
   \UNTIL{converged}
\end{algorithmic}
\end{algorithm}

\label{sec:framework_ldist}
\textbf{Distribution consistency $\ldist$} We aim to achieve the consistency between the joint distribution $\qq(\rvx, \rvz, \rvu)$ and $\pp(\rvx, \rvz, \rvu)$ through $T$ cost functions in $\ldist$. 
% This can be enforced by using any tractable cost functions defined over joint, conditional or marginal distributions of $\qq$ and $\pp$, through divergence of probability distribution or loss function of corresponding samples.
With the proposed inference model in Sec.~\ref{sec:framework_inference}, we could decompose our $\ldist$ into two primary components:
(i)~\emph{Enforcing $\qq(\rvx, \rvz, \rvu) = \pp(\rvx, \rvz, \rvu)$}~Many previous works\cite{kingma-vae,wae,ali,bigan} have been proposed to learn a latent variable generative model to model the joint distribution, any tractable objective can be utilized here, we adopt the \emph{ELBO} as the default choice.
(ii)~\emph{Enforcing $\qq(\rvz) = \pp(\rvz)$}~ The reason that we explicitly include this objective in $\ldist$ is due to our $\pp$-dependent parametrization of $\qq(\rvz \mid \rvx) \propto \pp(\rvz) \prod_{i=1}^{N} \frac{\qq(\rvz \mid \rvx_i)}{\pp(\rvz)}$. Thus we explicitly enforce the consistency between the  induced marginal distribution $\qq(\rvz) \equiv \E_{\qq} \qq(\rvz \mid \rvx)$ and $\pp(\rvz)$. Tractable divergence estimators for minimizing $ C_T\left(q_\rvphi(\rvz) \;\|\; p_\rvtheta(\rvz) \right)$ have been proposed and analyzed in previous works,
% we use the Maximum-Mean Discrepancy~\cite{mmd_nips} and the estimator in ~\cite{hfvae} in this work. We present the complete distribution consistency objective here
\begin{equation}
\begin{aligned}
\ldist = \sum_{t=1}^{T-1} \alpha_t C_t(q_\rvphi \;\|\; p_\rvtheta) + \alpha_T C_T\left(q_\rvphi(\rvz) \;\|\; p_\rvtheta(\rvz)\right) \, .
\end{aligned}
\end{equation}

With the distribution consistency objective and the compositional inference model introduced in Sec.~\ref{sec:framework_inference}, we could train the latent variable generative model in a weakly/semi-supervised manner in terms of (i)~incomplete data where $\rX$ is partially observed (e.g. missing attributes in feature vectors, or missing a modality in multi-modal dataset), and (ii)~partial known dependency structure in $\rrM^q$ and $\rrM^p$.

\textbf{Structure learning}~
In this work, we show that our proposed framework is capable of learning the structure of Bayesian network $\gG^q$ and $\gG^p$ based on many existing structure learning methods efficiently, with \emph{gradient-based} optimization techniques, which avoids searching over the discrete super-exponential space.
Specifically, we show that our proposed framework can
(i) represent the assumptions made about the structure of the true data distribution in the form of a set of structural regularization in the form of Bayesian networks $\{\gG^k\}$ as the \emph{explicit inductive bias}.
A score-based structure learning objective is then introduced where $\lreg$ plays a vital role in scoring each candidate structure;
and (ii) utilize the non-stationary data from multiple environments~\cite{nonlinear_ica_tcl,irm,mila_metacausal} as additional observed random variables.
We show the score-based structure learning objective as below
\begin{equation}
\label{eq:loss_score}
\min_{\rvm^q, \rvm^p} \Ls_{\mathrm{score}} = \ldist + \lreg + \Ls_{\mathrm{sparsity}} \, . 
\end{equation}
We assume a jointly factorized Bernoulli distribution prior for structure variable $\rrM^q$ and $\rrM^p$, parametrized by $\bm{\rho^q}$ and $\bm{\rho^p}$.
We use the gumbel-softmax trick proposed by \cite{gumbel-softmax,concrete_dist,gumbel_family} as gradient estimators.
Following the Bayesian Structural EM~\cite{structural_em,mib_hidden} algorithm, we optimize the model alternatively between optimizing distributions $\Ls(\qq, \pp)$ and structure variables $\Ls_{\mathrm{score}}(\rvm^q, \rvm^p)$.
We present the full algorithm to train the proposed generative model with optional structure learning procedure in Alg.~\ref{alg:learning}.

\begin{table}[t]
\caption{Distribution consistency objectives $\ldist$}
\label{tb:pool_ldist}
% \vskip 0.15in
\begin{center}
\begin{small}
\begin{sc}
\begin{tabular}{cc}
\toprule
$C$ &  $definition$ \\
\midrule
$C_0(\rvx, \rvz, \rvu)$    &  $\kl{\qq}{\pp}$\\
$C_1(\rvx, \rvu)$    &  $-\Ls_{\mathrm{ELBO}}(\qq(\rvx, \rvu), \pp(\rvx, \rvu))$\\
$C_2(\rvx)$    &  $\js{\qq(\rvx)}{\pp(\rvx)}$\\
$C_3(\rvz)$    &  $\kl{\qq(\rvz)}{\pp(\rvz)}$\\
$C_4(\rvx_i, \rvz)$    &  $\kl{\qq(\rvx_i, \rvz)}{\pp(\rvx_i, \rvz)}$\\
\bottomrule
\end{tabular}
\end{sc}
\end{small}
\end{center}
\vskip -0.1in
\end{table}

\begin{table*}[t]
\caption{A unified view of \{single/multi\}-\{modal/domain/view\} models. $C_i$ is referred to as the definition in Table.~\ref{tb:pool_ldist}, $\gG$ is referred to as the Bayesian networks in Figure.~\ref{fig:bn_vae},~\ref{fig:bn_mvae}. We use $N$ to denote the number of views/domains/modals. We use \textcircled{1} to denote \textit{shared/private latent space decomposition}, and use \textcircled{2} to denote \textit{dependency structure learning}.  Please see Appendix.~\ref{ap:framework} for the full table.}
\label{tb:unified_models}
% \vskip 0.15in
\begin{center}
\begin{small}
\begin{sc}
\begin{tabular}{lccccccl}
\toprule
Models & $N$ & \textcircled{1} & \textcircled{2} & $\gG^q$ & $\gG^p$ &  $\ldist$ & $\lreg$ \\
\midrule
VAE    & $1$ & $\times$ & $\times$ & $\left[\gG^q_{\mathrm{single}}\right]$ & $\left[\gG^p_{\mathrm{single}}\right]$ & $[1, C_1]$ & $[]$\\
\midrule
GAN    & $1$ & $\times$ & $\times$ & [] & $\gG^p_{\mathrm{single}}$ & $[1, C_2]$ & $[]$\\
\midrule
InfoGAN & $1$ & $\times$ & $\times$ & [] & $\gG^p_{\mathrm{single}}$ & $[1, C_2]$ & $[1, \gG^{\mathrm{InfoGAN}}]$\\
\midrule
$\beta$-VAE & $1$ & $\times$ & $\times$ & $\left[\gG^q_{\mathrm{single}}\right]$ & $\left[\gG^p_{\mathrm{single}}\right]$ & $[1, C_1], [\beta-1, C_3]$ & $[\beta - 1, \gG^\emptyset]$\\
\midrule
$\beta$-TCVAE & $1$ & $\times$ & $\times$ & $\left[\gG^q_{\mathrm{single}}\right]$ & $\left[\gG^p_{\mathrm{single}}\right]$ & $[1, C_1], [\alpha_2, C_2]$ & $[\beta, \gG^p]$\\
\midrule
% BiVCCA & $2$ & $\times$ & $\times$ & $\left[\gG^q_{\mathrm{marginal}}\right]$ & $\left[\gG^p_{\mathrm{joint}}\right]$ & $[\alpha_i, C_4(\rvx_i, \rvz)]]$ & $[]$\\
% \midrule
JMVAE  & $2$ & $\times$ & $\times$ & $\left[\gG^q_{\mathrm{joint}}\right]$ & $\left[\gG^p_{\mathrm{joint}}\right]$ & $[1, C_1]$& $[\beta_i, \gG^{\mathrm{str}}_{\mathrm{cross}}(\rvx_i)]$\\
\midrule
% TELBO  & $2$ & $\times$ & $\times$ & $\left[\gG^q_{\mathrm{joint}},\gG^q_{\mathrm{marginal}}\right]$ & $\left[\gG^p_{\mathrm{joint}}\right]$ &  $[1, C_1]$ & $[\beta_i, \gG^{\mathrm{str}}_{\mathrm{marginal}}(\rvx_i)]$\\
% \midrule
MVAE   & $N$ & $\times$ & $\times$ & $\left[\gG^q_{\mathrm{joint}},\gG^q_{\mathrm{marginal}}\right]$ & $\left[\gG^p_{\mathrm{joint}}\right]$ & $[1, C_1]$ & $[\beta_i, \gG^{\mathrm{str}}_{\mathrm{marginal}}(\rvx_i)]$\\
\midrule
Wyner  & $2$ & $\checkmark$ & $\times$ & $\left[\gG^q_{\mathrm{joint}},\gG^q_{\mathrm{marginal}}\right]$ & $\left[\gG^p_{\mathrm{joint}}\right]$ & $[1, C_1]$ & $[\beta_i, \gG^{\mathrm{str}}_{\mathrm{cross}}(\rvx_i)],[\beta_i, \gG^{\mathrm{str}}_{\mathrm{private}}(\rvx_i)]$\\
\midrule
% DIVA  & $3$ & $\checkmark$ & $\times$ & $\left[\gG^q_{\mathrm{marginal}}\right]$ & $\left[\gG^p_{\mathrm{joint}}\right]$ & $[1, C_1]$ & $[\beta_i, \gG^{\mathrm{str}}_{\mathrm{private}}(\rvx_i)]$ \\
% \midrule
OURS-MM & $N$ & $\checkmark$ & $\checkmark$ & $\left[\gG^q_{\mathrm{full}} \right]$ & $\left[\gG^p_{\mathrm{full}} \right]$ & $[1, C_0]$ & $[\beta_i, \gG^{\mathrm{str}}_{\mathrm{cross}}(\{\rvx_i\})]$\\
% \midrule
% OURS-FR & $N$ & $\checkmark$ & $\checkmark$ & $\left[\gG^q_{\mathrm{full}} \right]$ & $\left[\gG^p_{\mathrm{full}} \right]$ & $[1, C_1]$ & $[\beta_i, \gG^{\mathrm{str}}_{\mathrm{fairness}}]$\\
% \midrule
% OURS-IRM & $N$ & $\checkmark$ & $\checkmark$ & $\left[\gG^q_{\mathrm{full}} \right]$ & $\left[\gG^p_{\mathrm{full}} \right]$ & $[1, C_1]$ & $[\beta_i, \gG^{\mathrm{str}}_{\mathrm{irm}}]$\\
\bottomrule
\end{tabular}
\end{sc}
\end{small}
\end{center}
\vskip -0.1in
\end{table*}

\section{Case study: Generative Data Modeling}
In this section, we show various types of generative data modeling can be viewed as a structured latent space learning problem, which can be addressed by our proposed framework in a principled way.
\begin{figure}[ht]
\vskip 0.2in
\begin{center}
\begin{tabular}{ccc}
\includegraphics[width=0.3\columnwidth]{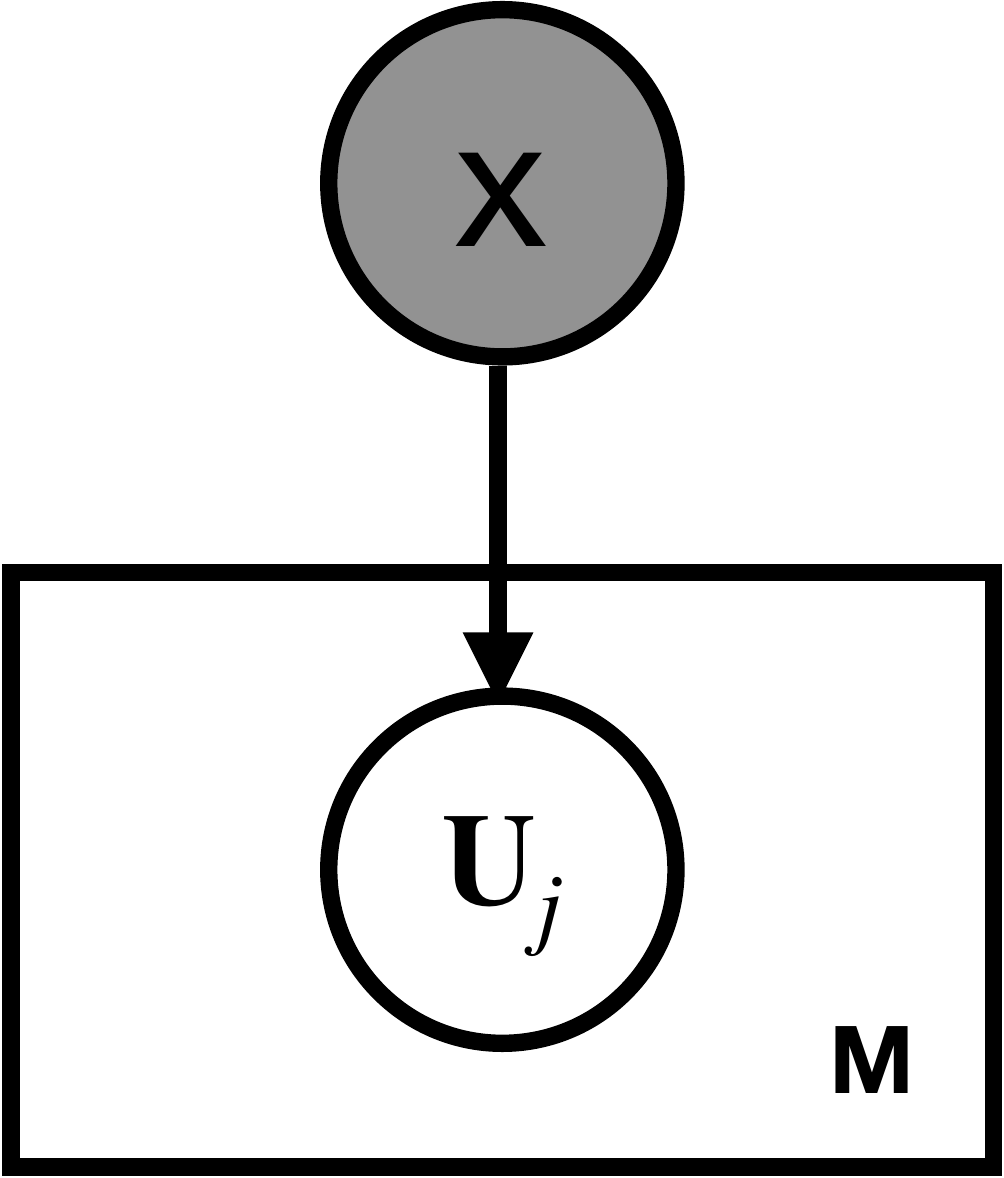} & 
\includegraphics[width=0.3\columnwidth]{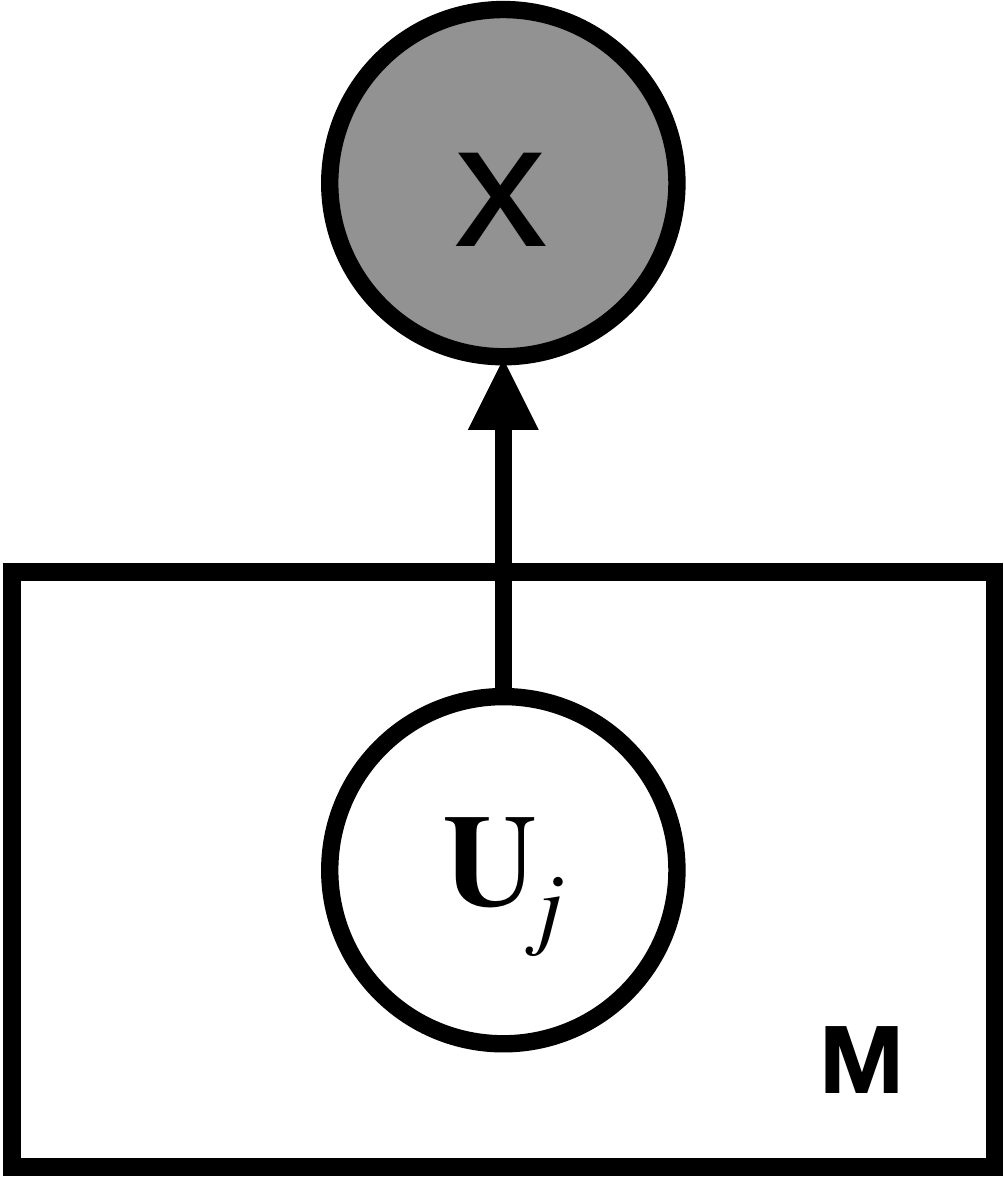} &
\includegraphics[width=0.3\columnwidth]{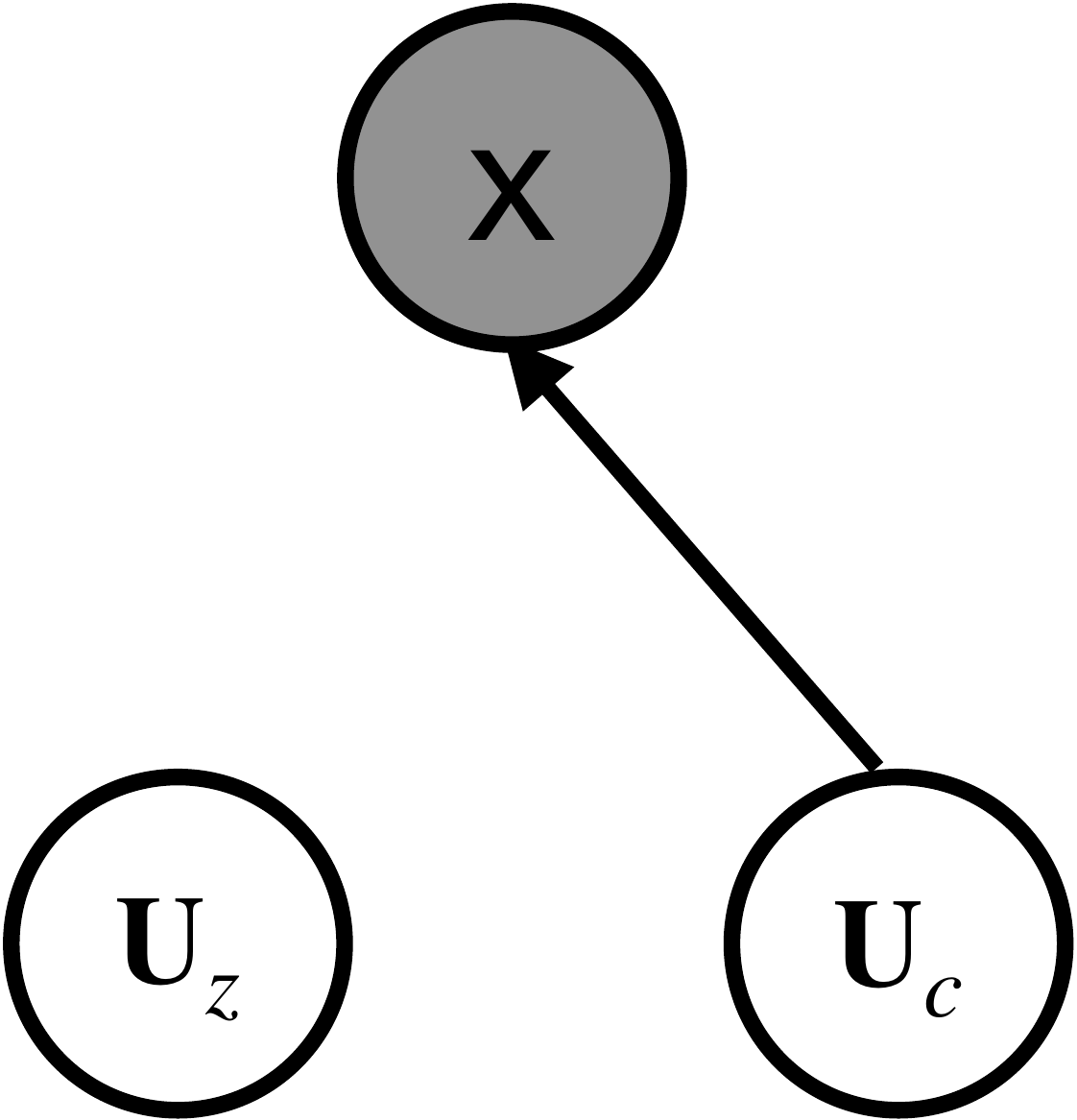}\\
$\gG^q_{\mathrm{single}}$ & $\gG^p_{\mathrm{single}}$ & $\gG^{\mathrm{str}}_{\mathrm{info}}$
\end{tabular}
\caption{Bayesian networks for single-modal models}
\label{fig:bn_vae}
\end{center}
\vskip -0.2in
\end{figure}

\begin{figure}[ht]
\vskip 0.2in
\begin{center}
\begin{tabular}{ccl}
\includegraphics[width=0.45\columnwidth]{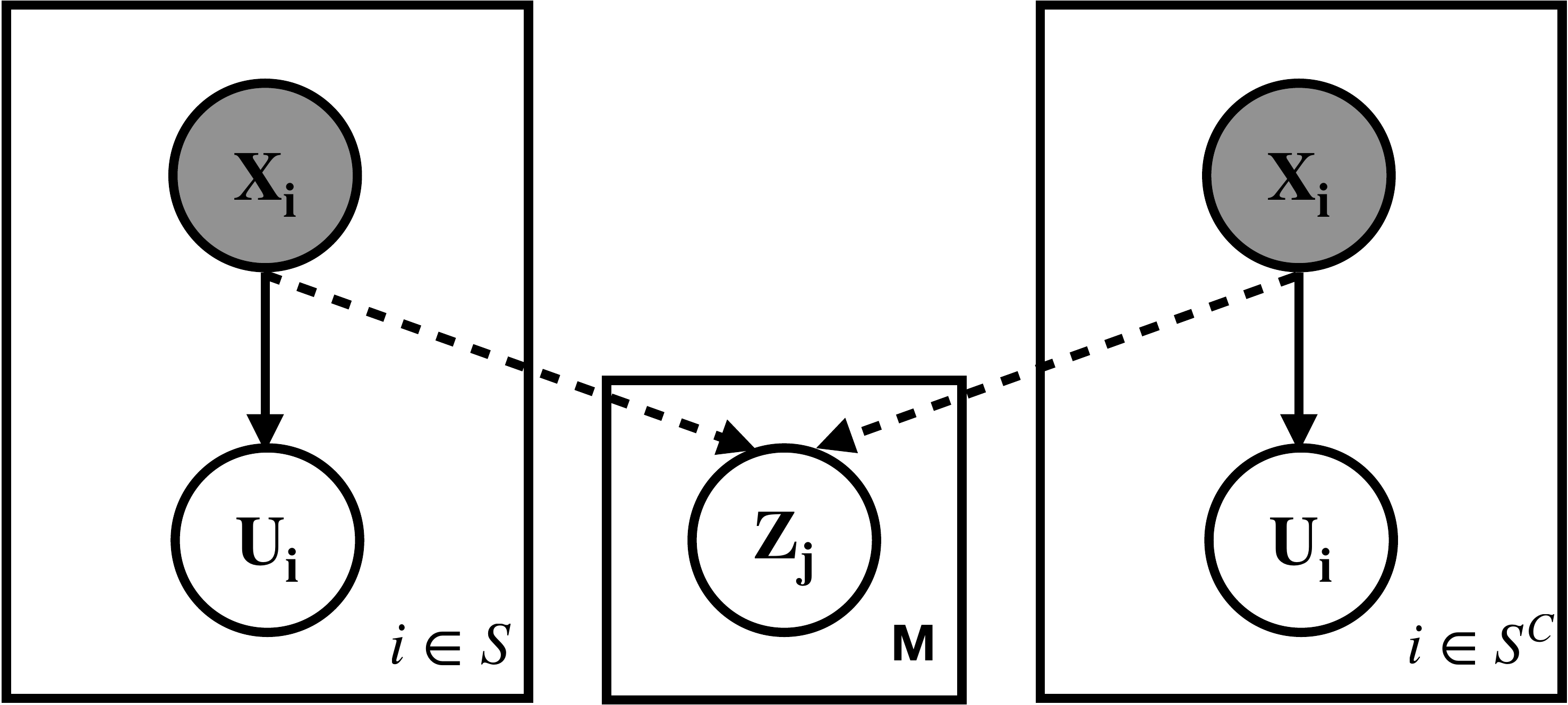} & 
\includegraphics[width=0.45\columnwidth]{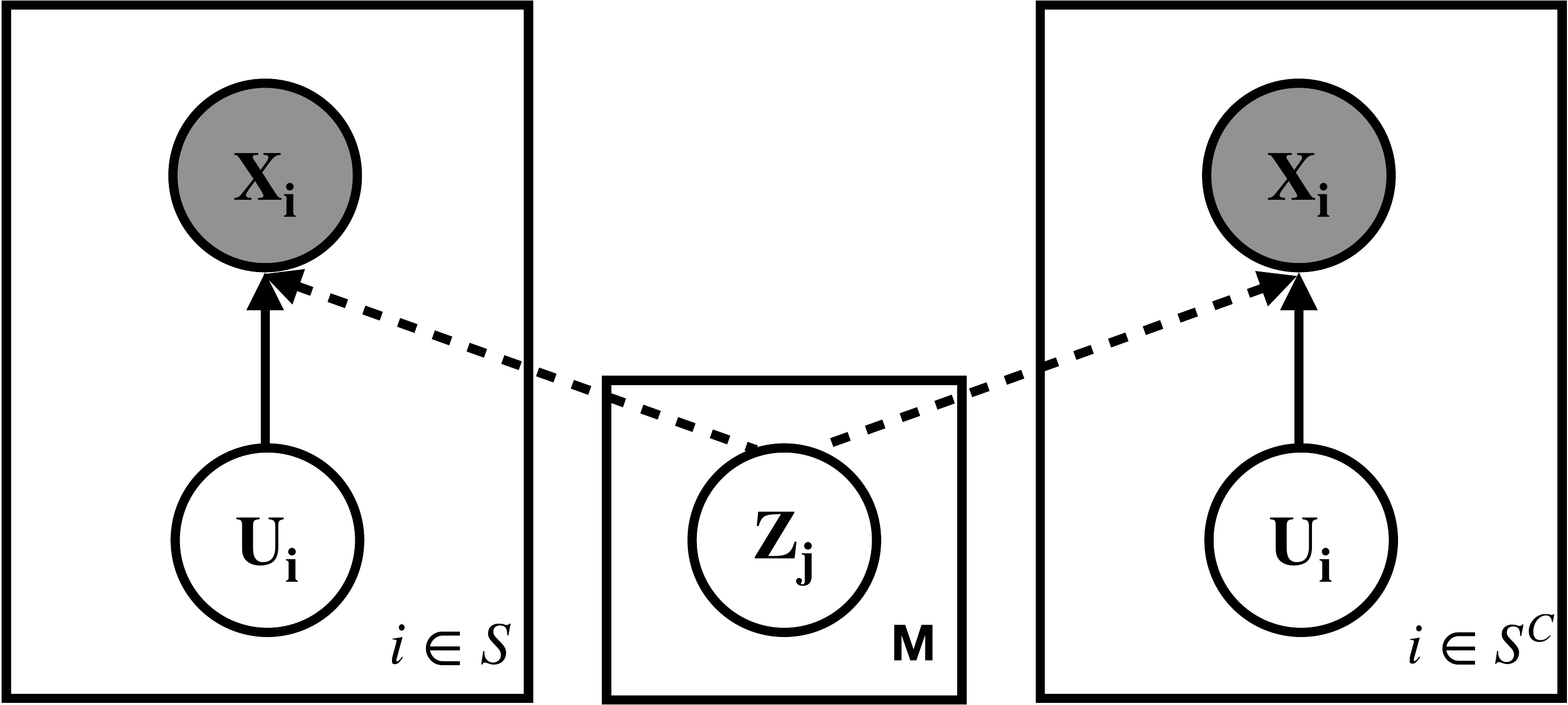} \\
$\gG^q_{\mathrm{full}}$ & $\gG^p_{\mathrm{full}}$\\
\includegraphics[width=0.45\columnwidth]{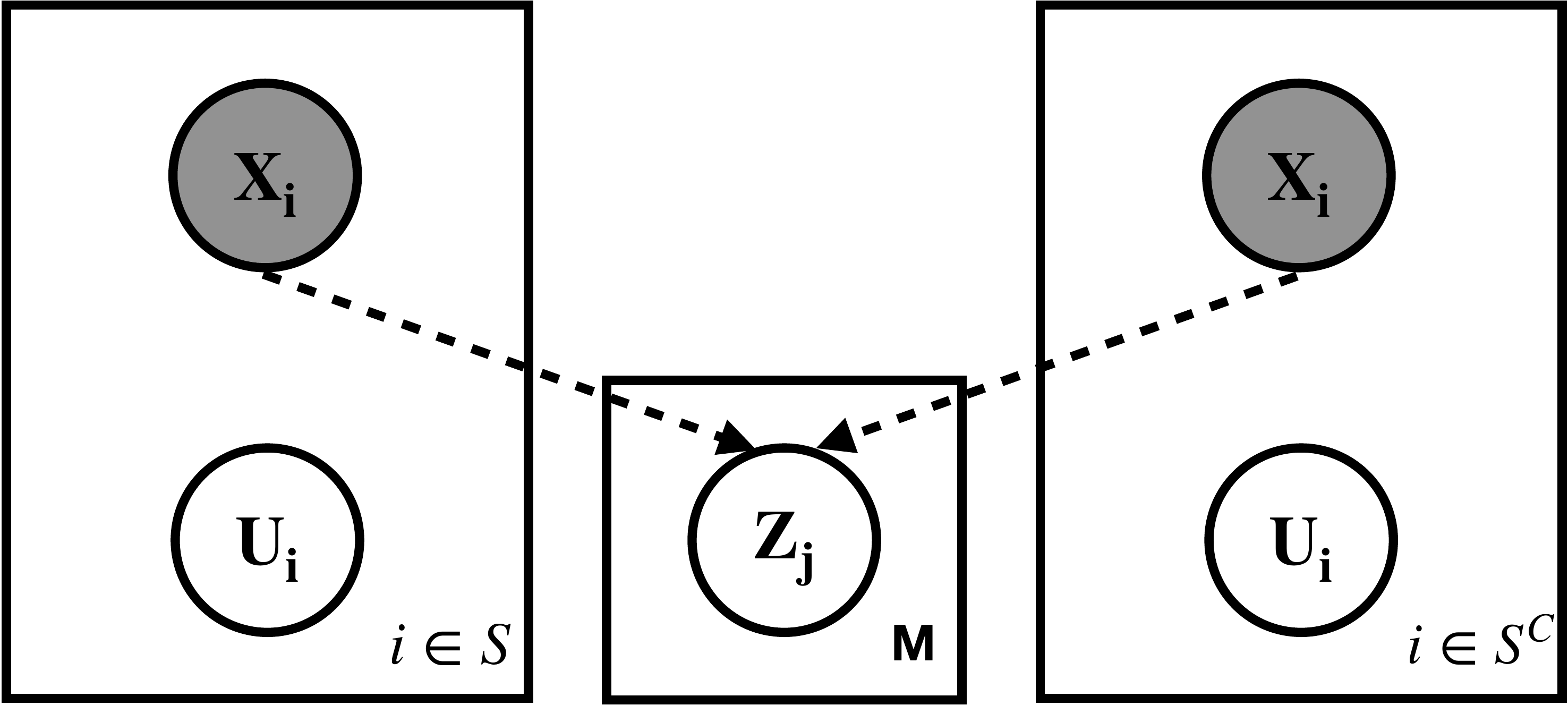} & 
\includegraphics[width=0.45\columnwidth]{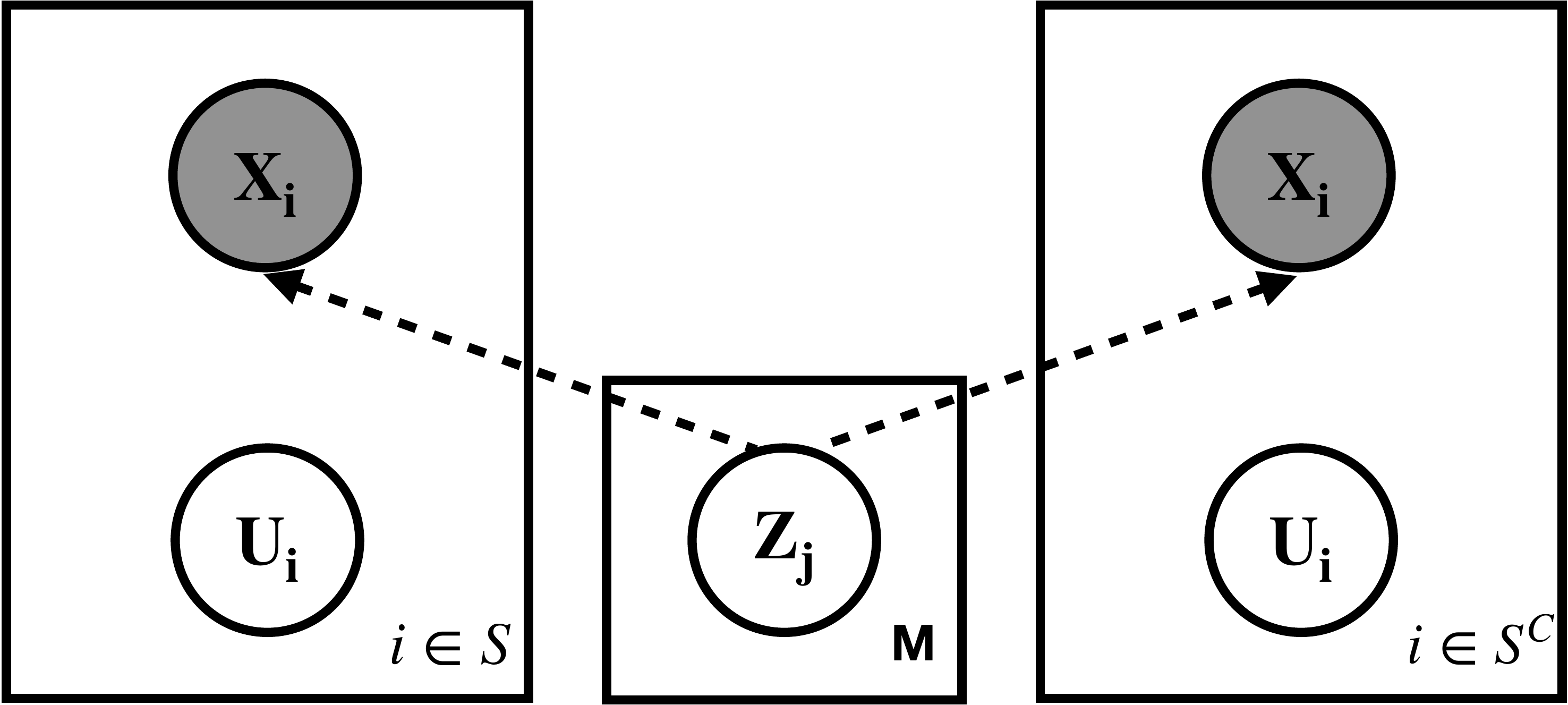} \\
$\gG^q_{\mathrm{joint}}$ & $\gG^p_{\mathrm{joint}}$\\
\includegraphics[width=0.45\columnwidth]{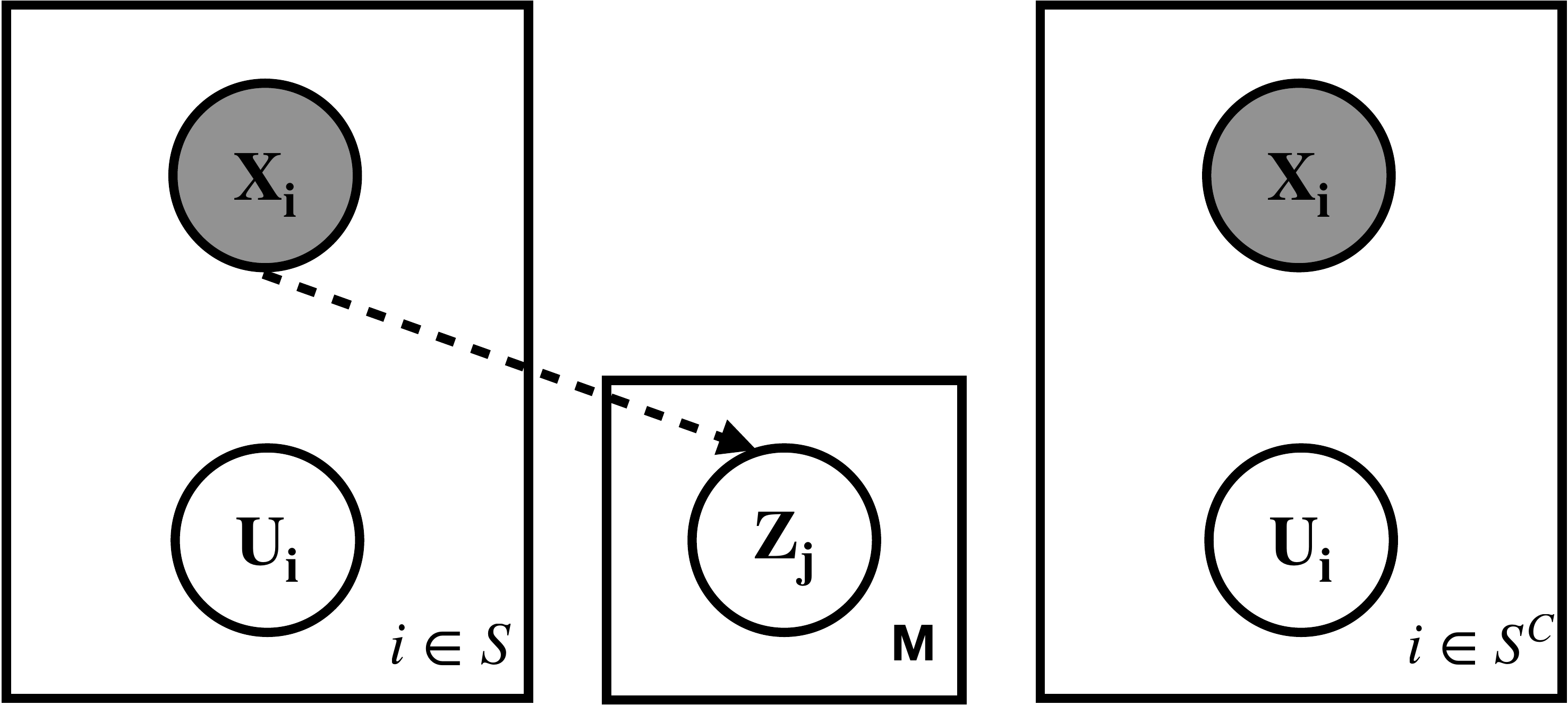} &
\includegraphics[width=0.45\columnwidth]{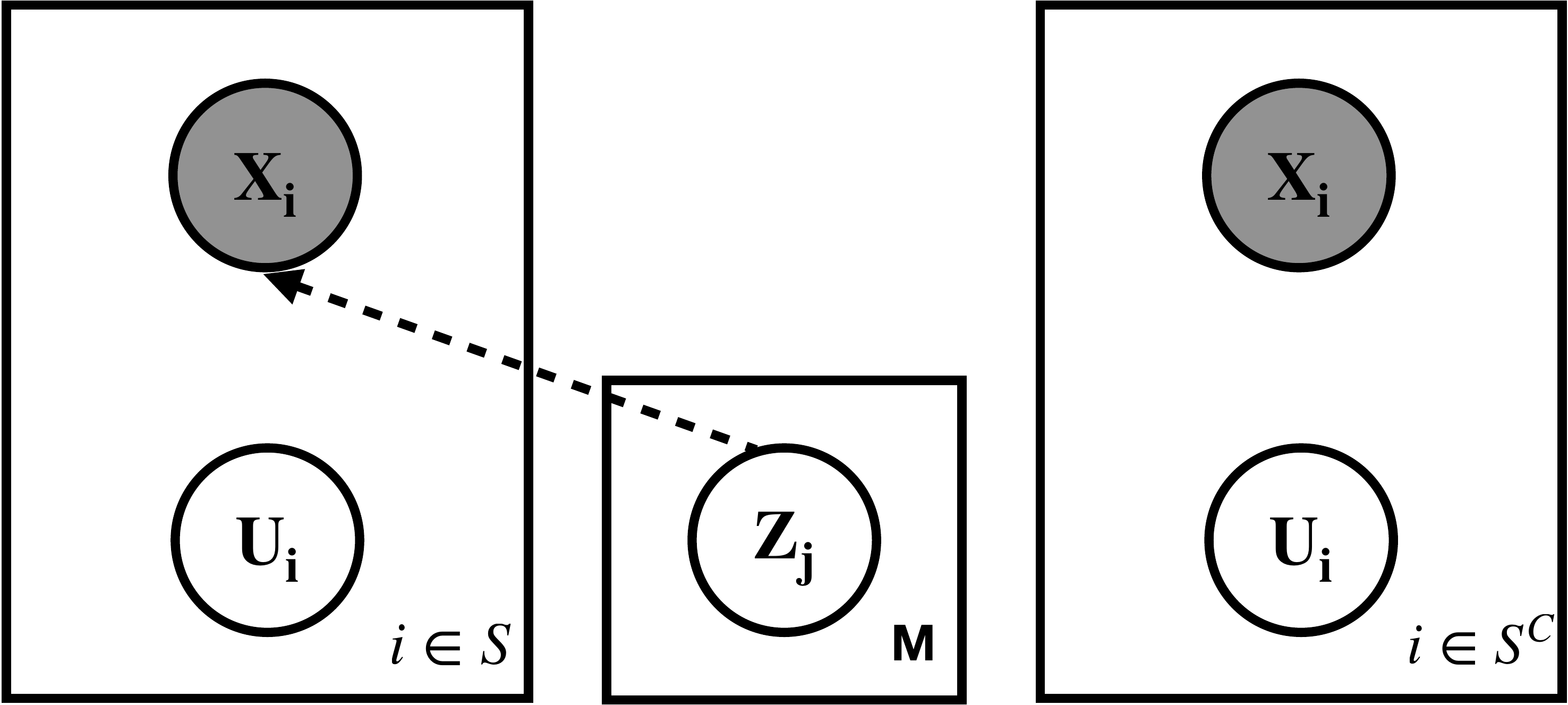}\\
$\gG^q_{\mathrm{marginal}}$ & $\gG^{p}_{\mathrm{marginal}}$\\
\includegraphics[width=0.45\columnwidth]{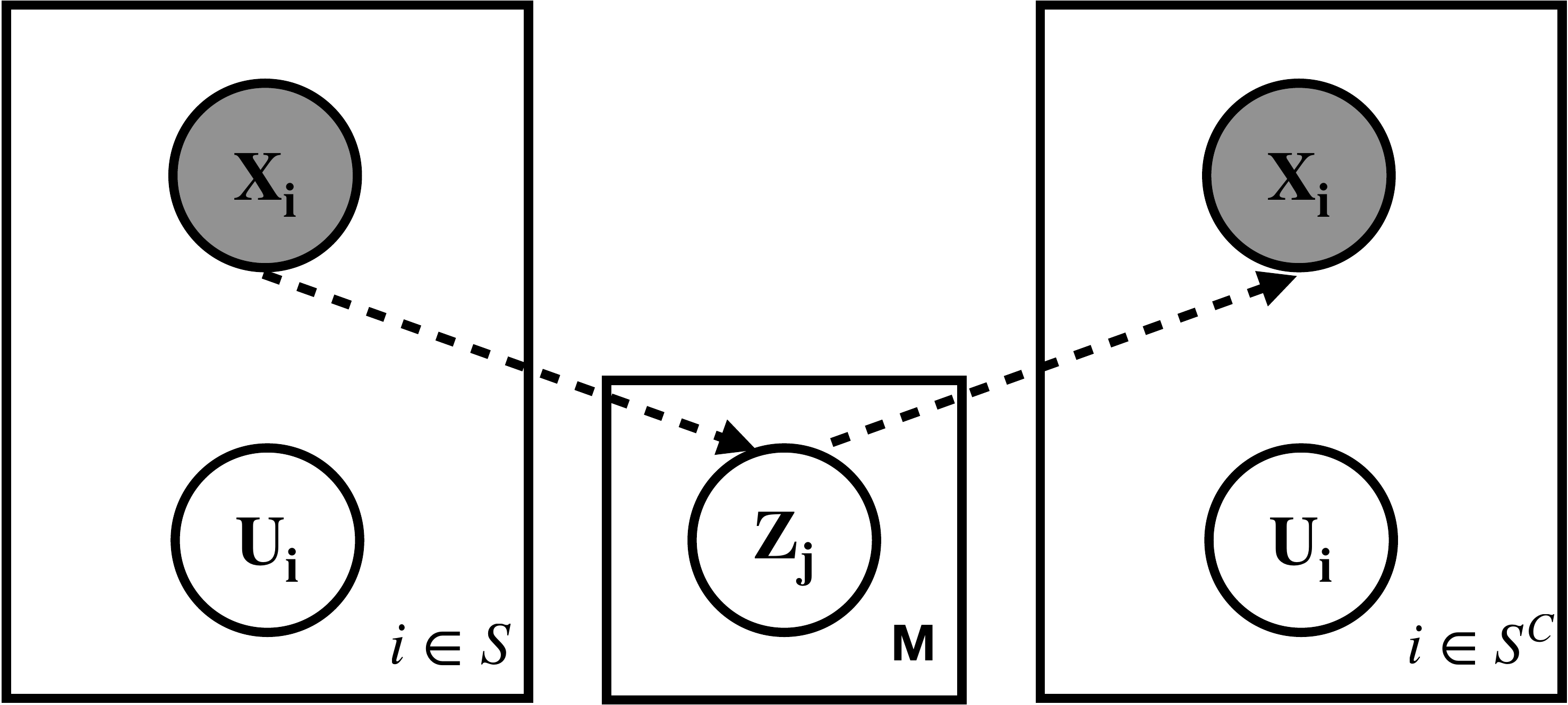} &
\includegraphics[width=0.45\columnwidth]{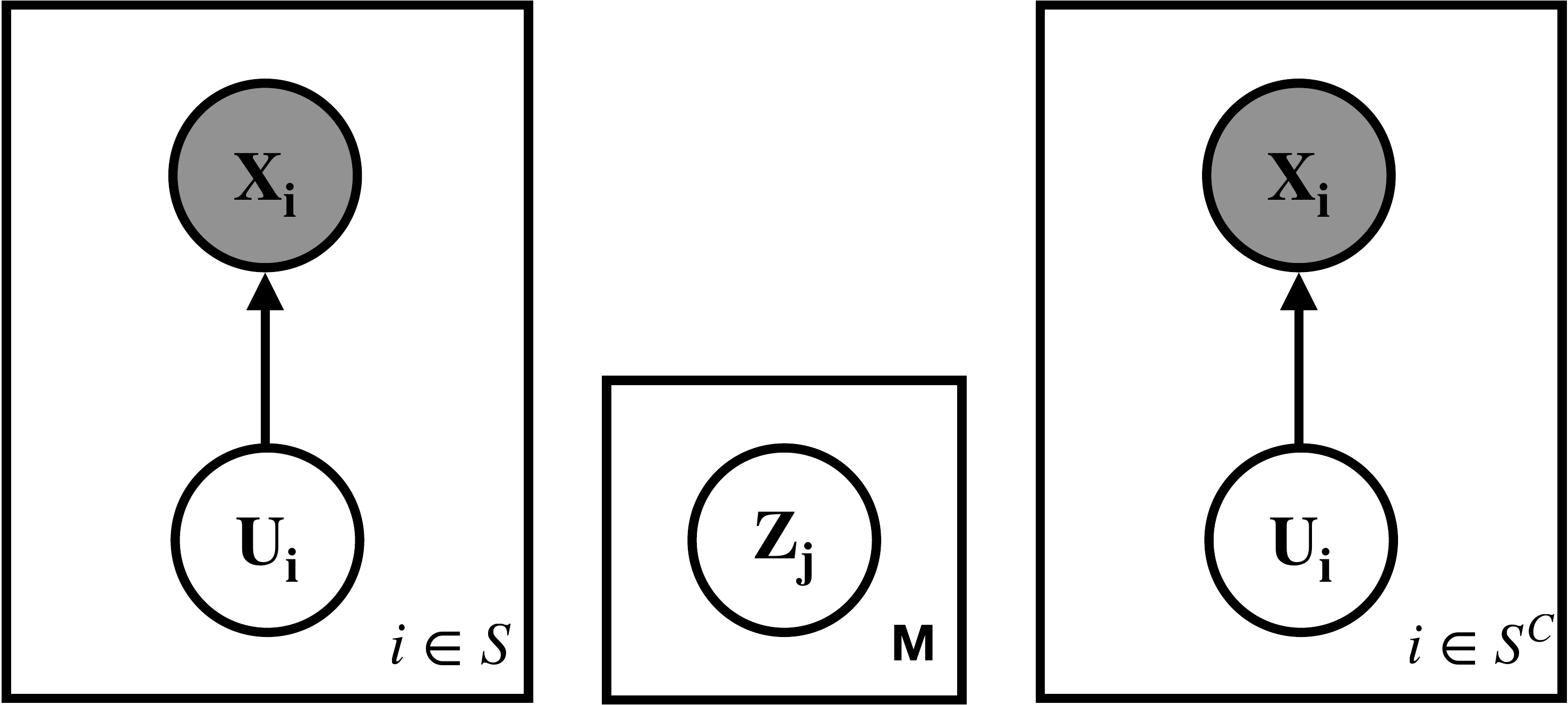}\\
$\gG^{\mathrm{str}}_{\mathrm{cross}}$ & $\gG^{\mathrm{str}}_{\mathrm{private}}$\\
\end{tabular}
\caption{Bayesian networks of various inference, generation models and structural regularizations in multi-modal/domain/view setting.}
\label{fig:bn_mvae}
\end{center}
\vskip -0.2in
\end{figure}

\subsection{Single-modal generative model}
\label{sec:single_model_model}
\textbf{Framework specification}~
In single-modal data generative modeling setting, we have $N = 1$ observed variable $\rX \equiv \left[\rX_{1}\right]$ which could be in image, text or other modalities, and we only incorporate private latent variables $\rU$.
% as our latent variables since we are interested in modeling the modality/domain-level structure using $\rrM$, instead of low-level data features like pixels in image data.
We abuse the notation a little by assuming $M$ latent variables $\rU \equiv \left[\rU_1, \rU_2, \ldots, \rU_M\right]$\footnote{because we can define arbitrary dimension for $\rU$.}.

\textbf{A unified view}
We show that our proposed model unifies many existing generative models.
% One of the most interesting directions in generative models research is learning disentangled representation of data, which captures semantically meaningful independent factors of variations.
We show that we can impose disentanglement as a special case of the structural regularization in latent space to obtain different existing disentangled representation learning methods.
We summarize how existing generative models can be unified within our proposed information-theoretic framework in table~\ref{tb:unified_models}.
As an interesting example, we show that we can derive the $\beta$-vae objective with $\Ls = C_1 + (\beta - 1) C_3 + (\beta - 1)\lreg(\gG^{\emptyset})$, 
where we impose the structural regularization $(\beta - 1)\projkl{q_\rvphi}{\gG^{\emptyset}}$.
% Thus we show that the $\beta$-vae objective is equivalent to imposing another empty Bayesian network structure in the latent space which implies the independent latent factors.
In this way, we also established connections to the results in~\cite{beta_vae_prior,ddvae} that $\beta$-vae is optimizing \emph{ELBO} with a $\qq$-dependent implicit prior $r(\rvu) \propto \qq(\rvu)^{1 - \beta}\pp(\rvu)^\beta$,
we achieve this in a symmetric way by using a $\pp$-dependent posterior $\qq(\rvz \mid \rvx) \propto \pp(\rvz) \prod_{i=1}^{N} \frac{\qq(\rvz \mid \rvx_i)}{\pp(\rvz)}$.
We further show that how we can unify other total-correlation based disentangled representation learning models~\cite{tcvae,hfvae,kimdisentangle} by explicitly imposing Bayesian structure $\gG^p$ as structural regularization.
We include detailed discussions and proofs in Appendix.~\ref{ap:sm}.

\subsection{Multi-modal/domain/view generative model}
\label{sec:multi_modal_model}
\textbf{Problem setup}
% Learning structured and robust representation from multivariate data in multi-modal, multi-domain or multi-view setting have been an emerging direction in machine learning.
% In this section, we illustrate our proposed model as a unified framework by casting multi-modal/domain/view representation learning problem as learning a generative model with shared latent space across different domains/views/modalities.
We represent the observed variables as $\rX_{1: N} \equiv \left[\rX_{1}, \rX_{2}, \ldots, \rX_{N}\right]$, where we have $N$ observed variables in different domains\footnote{We use the word domain to represent domain/modality/view.} and they might be statistically dependent.
We thus aim to learn latent factors $\rZ$ that explains the potential correlations among $\rX$.
Meanwhile, we also learn latent factors $\rU_i$ that explains the variations exclusive to one specific observed variable $\rX_i$.
In this way, we could achieve explicit control over the domain-dependent and domain-invariant latent factors.
For more details of the data generation process for this task and the model, please see Appendix.~\ref{ap:exp_mm}.

% We the aim to learn latent factors $\rZ$ that explains the potential correlation based on the Reichenbach's common cause principle, that when we condition on $\rZ$ we can decorrelate the observed data $\rX_i \independent \rX_j \mid \rZ$, which is implied by assuming a factorized distribution of $\pp(\rvx \mid \rvz)$ in our generation model $\pp$. Meanwhile, we also learn latent factors $\rU_i$ that explains the variations exclusive to one specific observed variable $\rX_i$, which could preserve robust and generalizable domain/modal-invariant information. With this proposed structured latent space, 

\textbf{A unified view}
We summarize the key results of unifying many existing multi-domain generative models in Table.~\ref{tb:unified_models}. We prove and discuss some interesting connections to related works in more details in Appendix.~\ref{ap:mm}, including BiVCCA~\cite{bivcca}, JMVAE~\cite{jmvae}, TELBO~\cite{telbo}, MVAE~\cite{mvae}, WynerVAE~\cite{wynervae}, DIVA~\cite{diva} and CorEx~\cite{corex,corex-hierarchical,corex-infosieve,corex-vae}.

\textbf{Framework specification}~
We present a specific implementation of our proposed framework for multi-domain generative modeling here.
We show that it generalizes some heuristics used in previous models and demonstrate its effectiveness in several standard multi-modal datasets.
We use $\ldist$ in Table~\ref{tb:pool_ldist} to learn consistent inference model and joint, marginal, conditional generation model over $\left( \rX, \rZ, \rU\right)$.
To embed multi-domain data into a shared latent space, we use the structural regularization that enforces Markov conditional independence structure $\rX^{\sS} \rightarrow \rZ \rightarrow \rX^{\sS^\complement}$.
This structural regularization can be represented by $\gG^{\mathrm{str}}_{\mathrm{cross}}$ in Figure~\ref{fig:bn_mvae}, where $\rX \equiv \left[\rX^{\sS}, \rX^{\sS^\complement} \right]$ is a random bi-partition of $\rX$.
%To embed multi-domain data into a shared latent space, we impose Markov conditional independence structure $\rX^{\sS} \rightarrow \rZ \rightarrow \rX^{\sS^\complement}$ indicated by Bayesian network $\gG^{\mathrm{str}}_{\mathrm{cross}}$ in Figure~\ref{fig:bn_mvae}, where $\rX \equiv \left[\rX^{\sS}, \rX^{\sS^\complement} \right]$ is a random bi-partition of $\rX$.
Then we show that the objective can be upper-bounded by $\Ls = \ldist + \lreg \le \Ls_{\rvx} + \Ls_{\rvu} + \Ls_{\rvz}$, where $\Ls_{\rvx}=-\E_{\qq(\rvz, \rvu \mid \rvx)}\log \pp(\rvx \mid \rvz, \rvu)$, $\Ls_{\rvu} = \E_{\qq(\rvx)}\kl{\qq(\rvu \mid \rvx)}{\pp(\rvu)}$ and $\Ls_{\rvz}=\sum_{i=0}^N \E_{\qq(\rvx)}\kl{\qq(\rvz \mid \rvx)}{\qq(\rvz \mid \rvx_i)}$.
We use $\qq(\rvz \mid \rvx_0) \equiv \pp(\rvz)$ for the simplicity of notations.
% We established a tractable upper-bound of the original objective as a sum of $\Ls_{\rvx}$, $\Ls_{\rvu}$ and $\Ls_{\rvz}$, where $\lreg$ is upper-bounded by $\Ls_{\rvz}$.
We further show that for each latent variable $\rZ_j$, $\Ls_{\rvz_j}$ term can be viewed as a generalized JS-divergence~\cite{jsd_abs_mean} among $\qq(\rvz_j \mid \rvx_i)$ for $i \in \left\{1,\ldots,N \right\}$ using geometric-mean weighted by $\rvm^q_j$, which can be seen as a generalization of the implicit prior used in $\beta$-vae as discussed in~\ref{sec:single_model_model}.
The detailed proof is presented in~ Appendix.~\ref{ap:mm}.
\begin{equation}
\begin{aligned}
% &\Ls_{\rvz} = \sum_{j=1}^M \Ls_{\rvz_j}, \quad \qq(\rvz_j \mid \rvx) \propto \prod_{i=0}^N \qq(\rvz_j \mid \rvx_i)^{\rvm^q_{ij}}\\
&\Ls_{\rvz_j} = D^{\rvm^q_j}_{\mathrm{JS}}\left(\qq(\rvz_j \mid \rvx_0), \qq(\rvz_j \mid \rvx_1), \ldots, \qq(\rvz_j \mid \rvx_N) \right)
% &D^{\mathrm{KL}^\ast}_{\mathrm{JS}}\left(q_0, \ldots, q_N \right) = \sum_{i=0}^N \gamma_i \kl{\qq(\rvz \mid \rvx)}{\qq(\rvz \mid \rvx_i)}\\
% &\sum_{i=0}^{N} \gamma_i = 1, \gamma_{0j} = 1 - \sum_{i=1}^N \rvm^q_{ij}, \quad \gamma_{ij} = \rvm^q_{ij} \; i > 0
\end{aligned}
\label{eq:jsd_objective}
\end{equation}

\begin{figure}[ht]
\vskip 0.2in
\begin{center}
\begin{tabular}{cc}
\includegraphics[width=0.44\columnwidth]{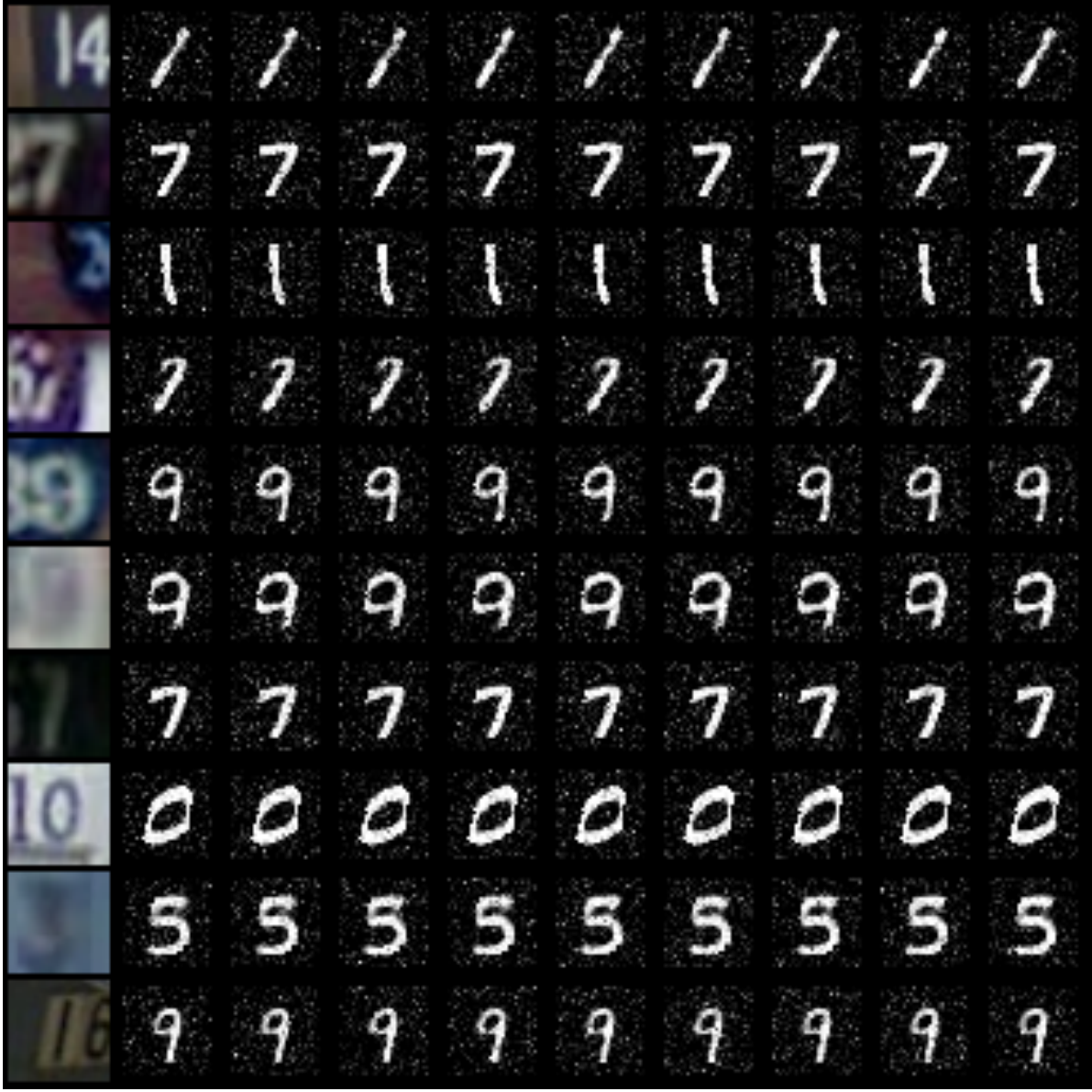} & 
\includegraphics[width=0.48\columnwidth]{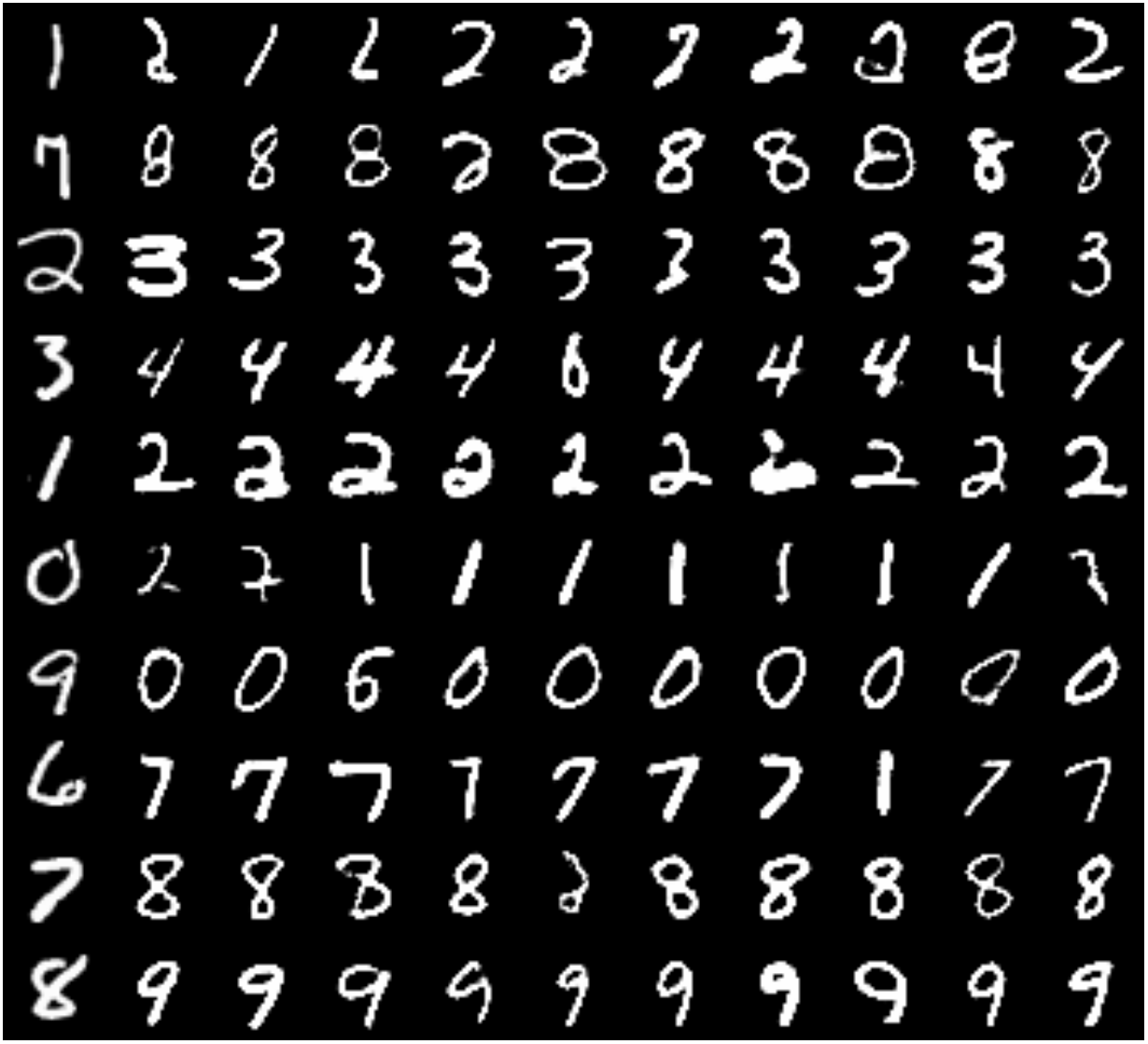}\\
SVHN $\rightarrow$ MNIST & MNIST $\rightarrow$ MNIST-Plus-1
\end{tabular}
\caption{Cross-domain generation samples. The leftmost column shows conditioned inputs.}
\label{fig:samples}
\end{center}
\vskip -0.2in
\end{figure}

\textbf{Experiment}
We validate the effectiveness of proposed model in multi-view/modal data modeling setting on bi-modal MNIST-Label, MNIST-SVHN and bi-view MNIST-MNIST-Plus-1 dataset. We show the generated samples in Figure~\ref{fig:samples}. 
The left panel in the figure contains the examples of MNIST-style samples generated by the model trained on MNIST-SVHN dataset when conditioned on SVHN data examples. 
We can observe that the model is using the shared latent variable $\rZ$ and private latent variables $\rU$ to successfully generate the MNIST-style samples of same digit as the SVHN inputs.
On the other hand, the right panel contains the examples of MNIST-style sample generated by the model trained on MNIST-Plus-1 dataset by conditioning on MNIST example.
We can observe that the model is successfully generating $m+1$ digit images when conditioned on $m$ digit input. 
More detailed results are included in the Appendix.~\ref{ap:exp_mm}.

\section{Case study: Fair Representation Learning}
\label{sec:fairness}
In this section, we show that fair representation learning can be viewed as a structured latent space learning problem, where we aim to learn a latent subspace that is invariant to sensitive attributes while informative about target label.
\begin{figure}[t]
\vskip 0.1in
\begin{center}
\begin{tabular}{ccc}
\includegraphics[width=0.3\columnwidth]{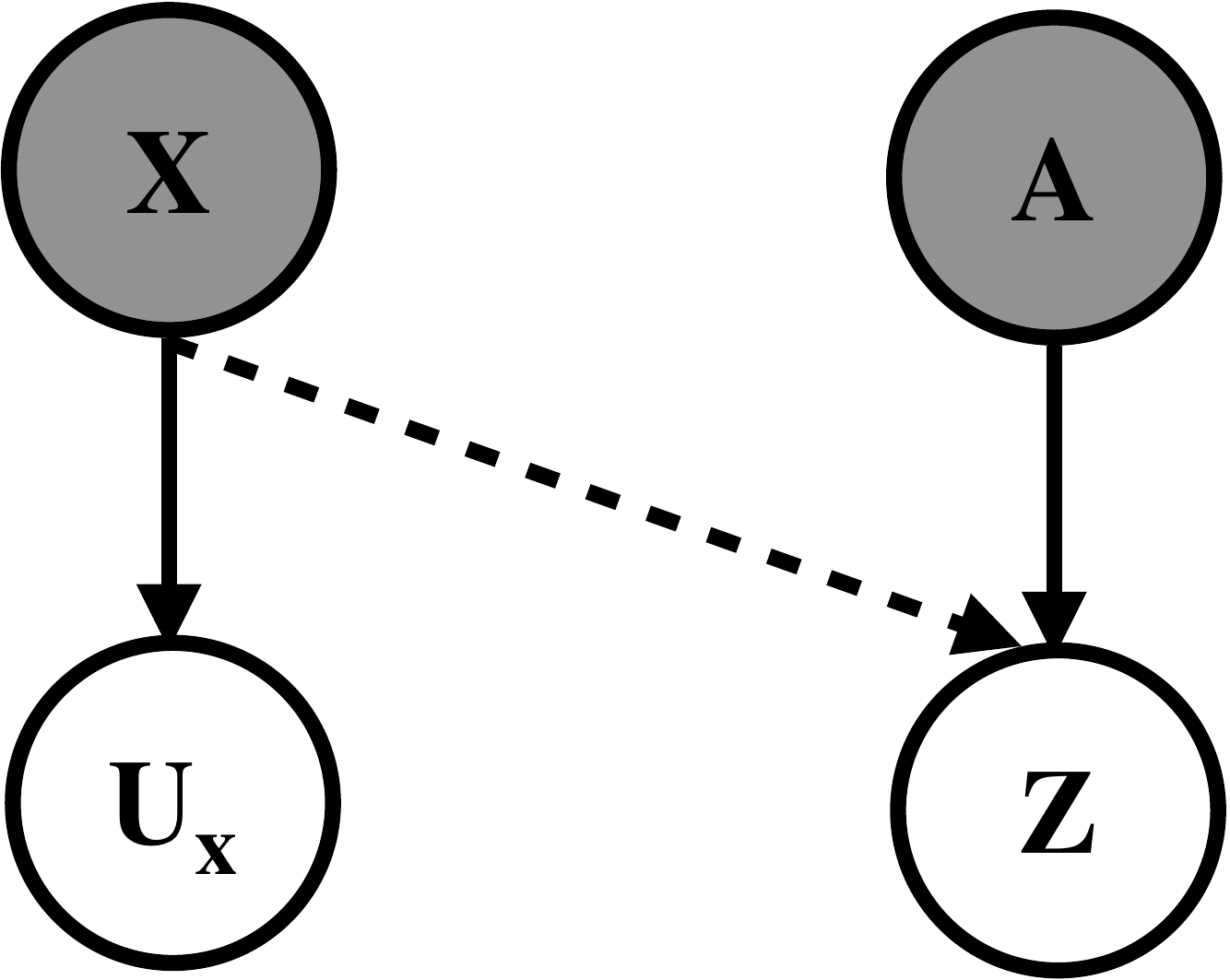} & 
\includegraphics[width=0.3\columnwidth]{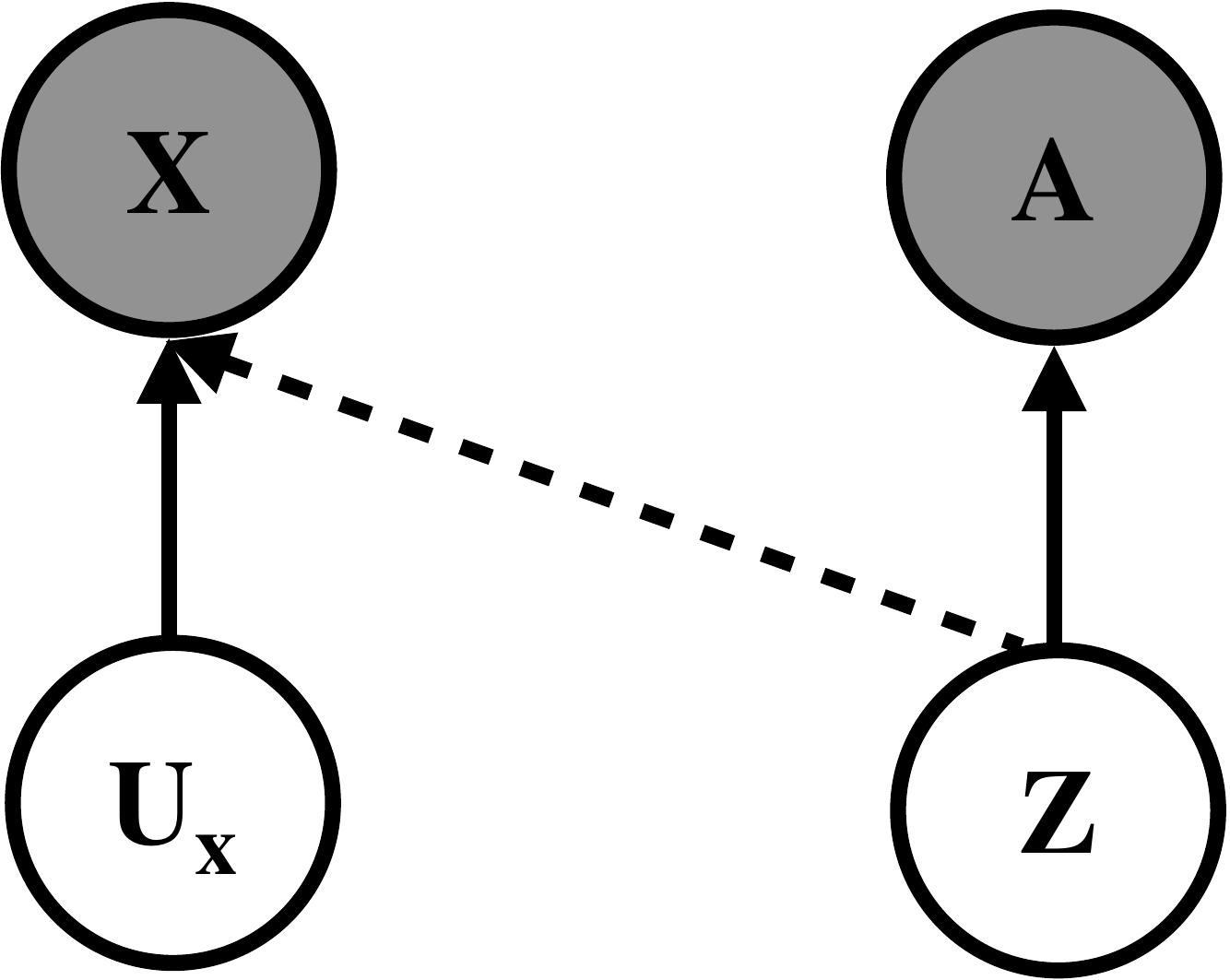} &
\includegraphics[width=0.3\columnwidth]{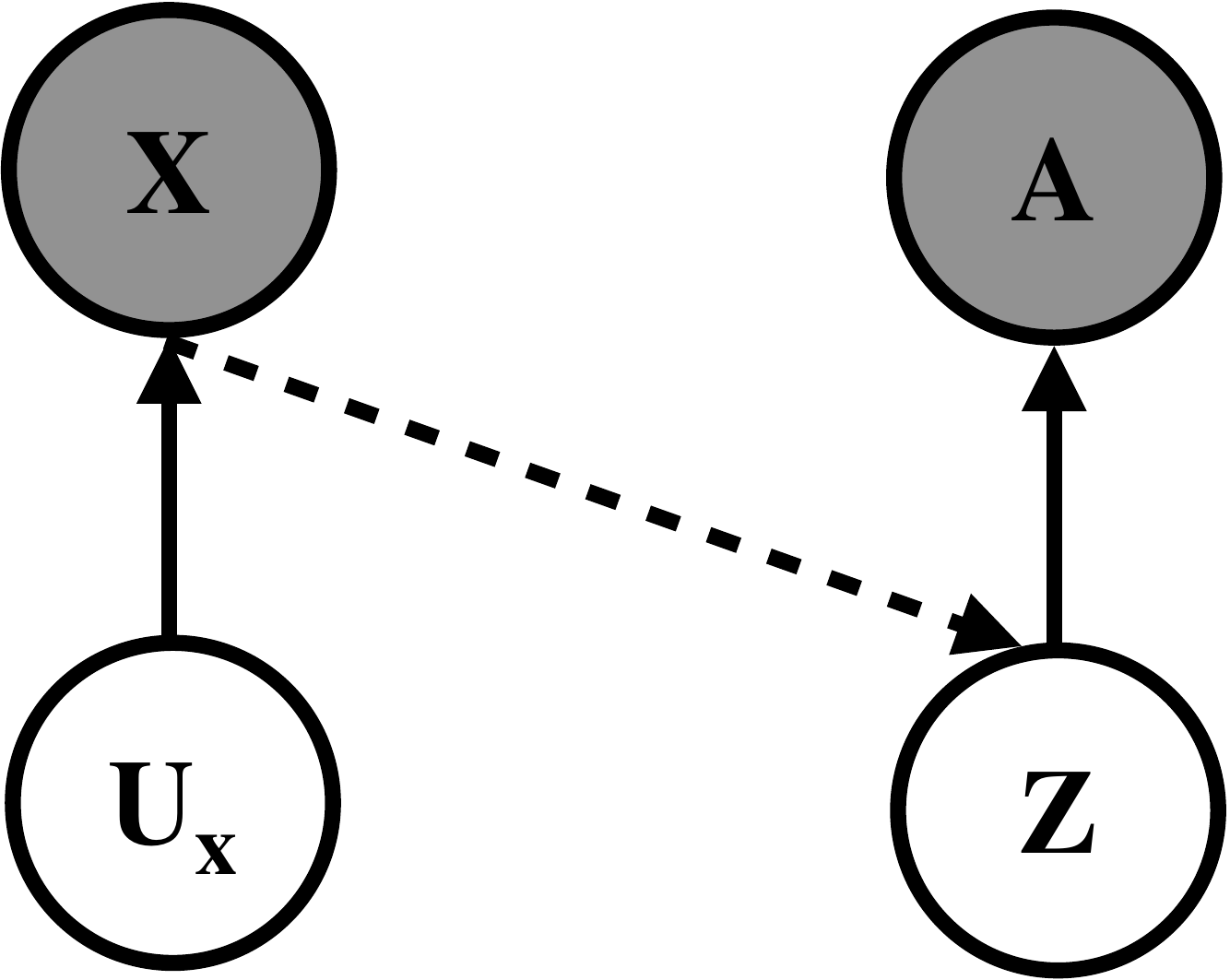} \\
$\gG^q$ & $\gG^p,\gG^{\mathrm{str}}_{\mathrm{invariant}}$ & $\gG^{\mathrm{str}}_{\mathrm{informative}}$\\
\end{tabular}
\caption{Bayesian networks for fair representation learning.}
\label{fig:bn_fair}
\end{center}
\vskip -0.2in
\end{figure}

\textbf{Problem setup}
We use $\left[\rX, \rA, \rY \right]$ to represent the observed variables, where the variable $\rX$ represents the multivariate raw observation like pixels of image sample, the variable $\rA$ represents the sensitive attributes, and the variable $\rY$ represents the target label to be predicted.
Following the same setting in previous works~\cite{fair_lagvae,fair_flexible}, the target label is not available during training phase.
A linear classifier using learned representation is trained to predict the held-out label $\rY$ in testing time.
% We train and utilize the model in a pre-processing manner to remove the sensitive information, then train 
We focus on the \emph{Difference of Equal Opportunity}~(DEO) notion in this work~\cite{fair_eod}.
For the details of the data generation process, please see Appendix.~\ref{ap:fairness}.

\textbf{Framework specification}
we learn a joint distribution over $\left[\rX, \rA, \rZ, \rU\right]$ with the framework proposed.
The shared latent variable $\rZ$ aims to explain the hidden correlation among $\rX$ and $\rZ$.
We also enforce two structural regularizations, represented by   
two Bayesian networks $\gG^{\mathrm{str}}_{\mathrm{invariant}}$ and  $\gG^{\mathrm{str}}_{\mathrm{informative}}$.  
The aim of $\gG^{\mathrm{str}}_{\mathrm{invariant}}$ is to learn the private latent variables $\rU_{\rvx}$ as the hidden factors that are invariant to the change of $\rZ$. 
Meanwhile, the aim of $\gG^{\mathrm{str}}_{\mathrm{informative}}$ is to 
preserve as much information about $\rX$ in $\rZ$.  
$\rrM^q$ and $\rrM^p$ are illustrated by $\gG^q$ and $\gG^p$ in Figure~\ref{fig:bn_fair} correspondingly.
Then we have the following learning objective
\begin{equation}
\begin{aligned}
&\Ls \le -\E_{\qq}\log \pp(\rvx, \rva \mid \rvz, \rvu) + \beta_2 \miq{\rvz}{\rvu} +\\
&(1+\beta_1) \E_{\qq} \kl{\qq(\rvz \mid \rvx, \rva)}{\pp(\rvz)} + const
%\E_{\qq(\rvx,\rva)} \kl{\qq(\rvz \mid \rvx, \rva)}{}
\end{aligned}
\label{eq:fair_objective}
\end{equation}
Please refer to Appendix.~\ref{ap:fairness} for the detailed derivation and discussion.

\begin{table}[t]
\caption{Fair representation learning results on German and Adult datasets.}
\label{tb:fair}
% \vskip 0.15in
\begin{center}
\begin{small}
\begin{sc}
\begin{tabular}{lcccr}
\toprule
Model & \multicolumn{2}{c}{Adult} & \multicolumn{2}{c}{German} \\
      & ACC & DEO & ACC & DEO\\ 
\midrule
Naive SVM    & $0.80$ & $0.09$ & $0.74 \pm 0.05$ & $0.12 \pm 0.05$ \\
SVM          & $0.79$ & $0.08$ & $0.74 \pm 0.03$ & $0.10 \pm 0.06$ \\
NN           & $0.84$ & $0.14$ & $0.74 \pm 0.04$ & $0.47 \pm 0.19$ \\
NN $+ \chi^{2}$    & $0.83$ & $0.03$ & $0.73 \pm 0.03$ & $0.25 \pm 0.14$ \\
FERM         & $0.77$ & $0.01$ & $0.73 \pm 0.04$ & $0.05 \pm 0.03$ \\
Ours-MMD     & $0.83$ & $0.02$ & $0.72 \pm 0.07$ & $0.07 \pm 0.09$ \\
Ours-TC     & $0.81$ & $0.02$ & $0.74 \pm 0.08$ & $0.08 \pm 0.14$ \\
Ours-MINE     & $0.79$ & $0.01$ & $0.70 \pm 0.11$ & $0.05 \pm 0.11$ \\
\bottomrule
\end{tabular}
\end{sc}
\end{small}
\end{center}
\vskip -0.1in
\end{table}
\textbf{Experiments} 
We investigate the performance of our derived objective on the UCI German credit dataset and the UCI Adult dataset. 
For estimating and minimizing $\miq{\rvz}{\rvu}$, we adopted MMD~\cite{mmd_nips}, total-correlation estimator in~\cite{hfvae} and MINE~\cite{mine} and summarize all results in Table~\ref{tb:fair}.
We report the classification accuracy (ACC) and the aforementioned DEO in the table.
The performances of all other baseline methods in the table are from~\cite{fair_continuous,fair_erm}. 
Please refer to Appendix.~\ref{ap:exp_fair} for more details.

\section{Case study: Out-of-Distribution Generalization}
\label{sec:irm}
\begin{figure}[ht]
\vskip 0.1in
\begin{center}
\begin{tabular}{ccc}
\includegraphics[width=0.3\columnwidth]{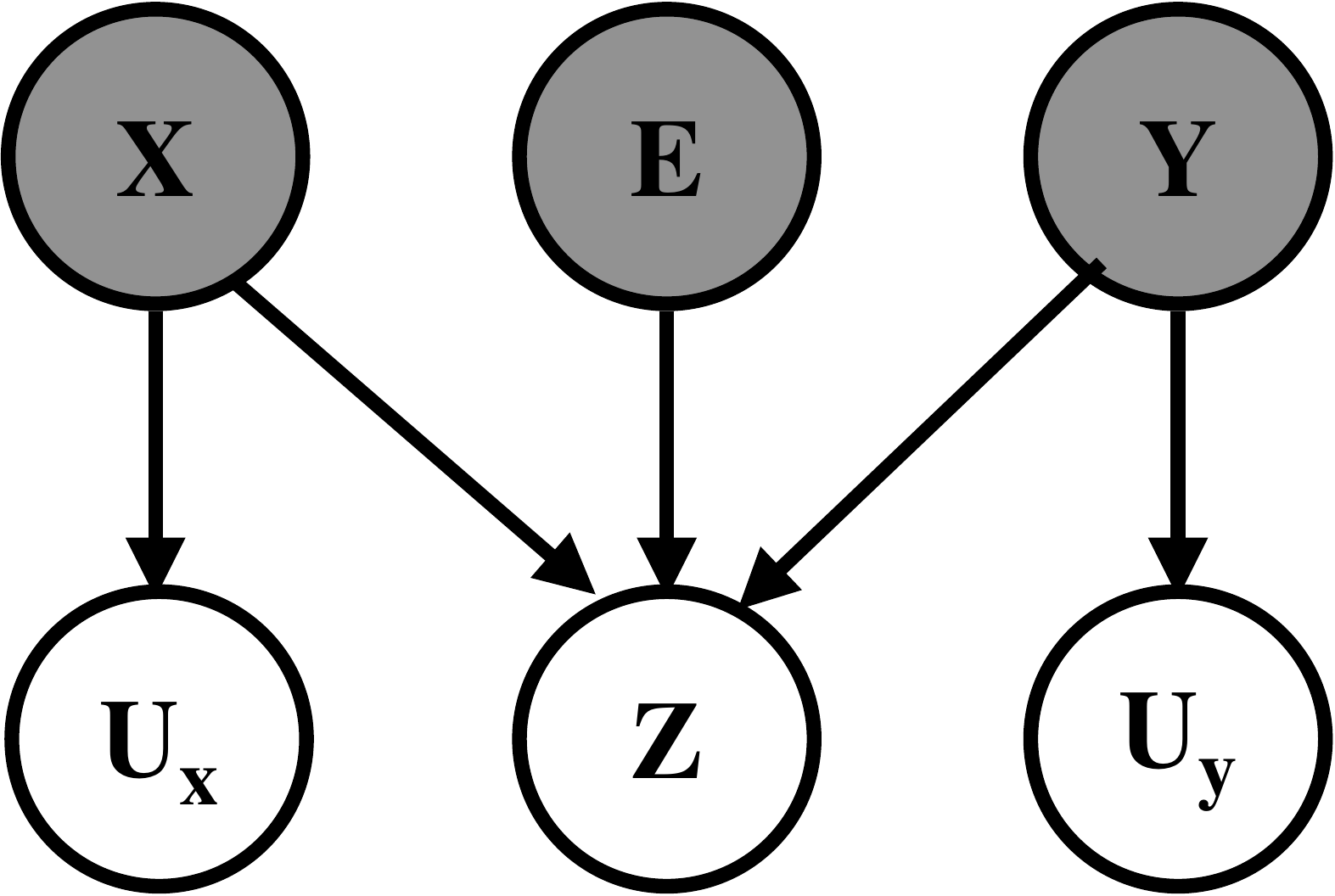} & 
\includegraphics[width=0.3\columnwidth]{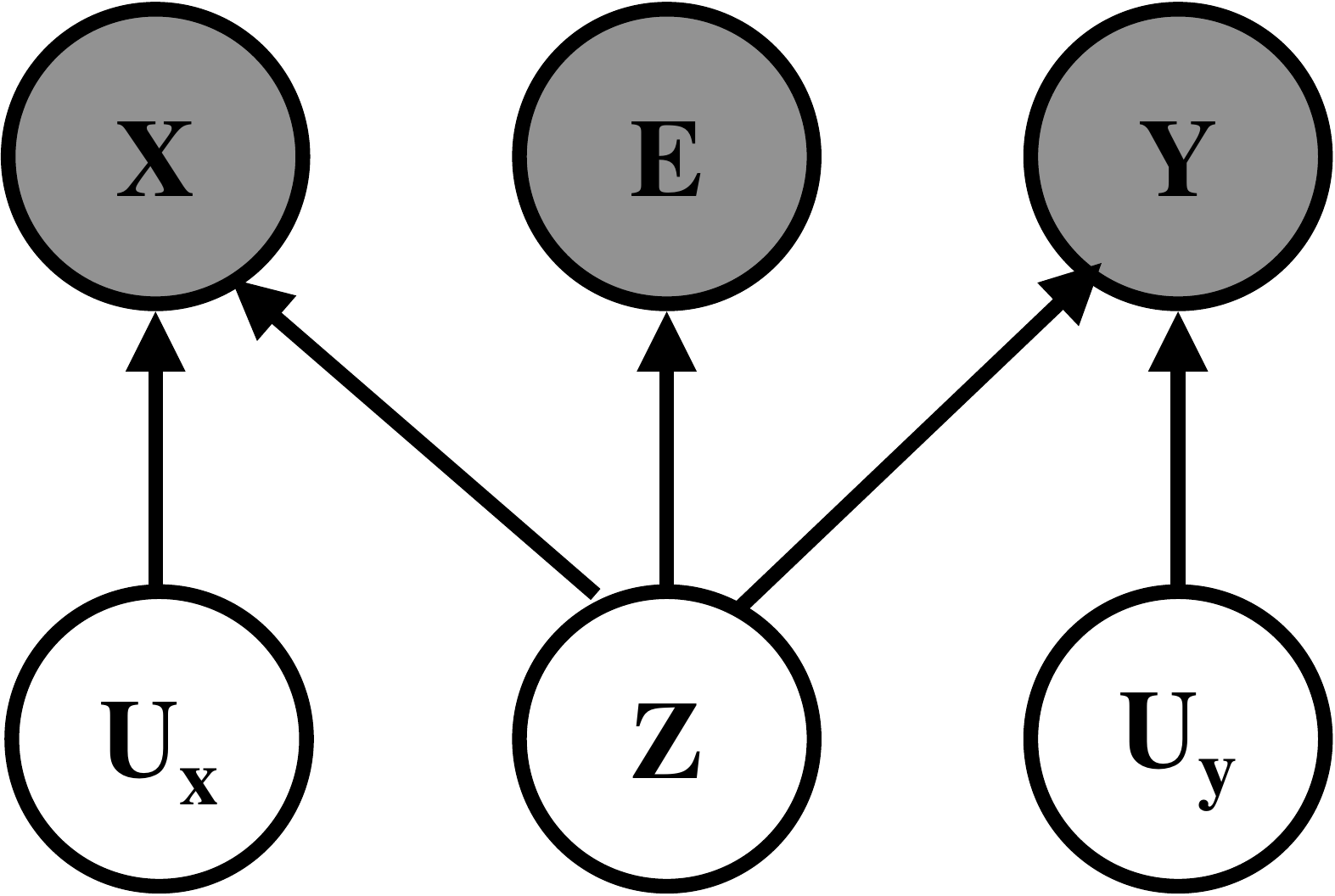} &
\includegraphics[width=0.3\columnwidth]{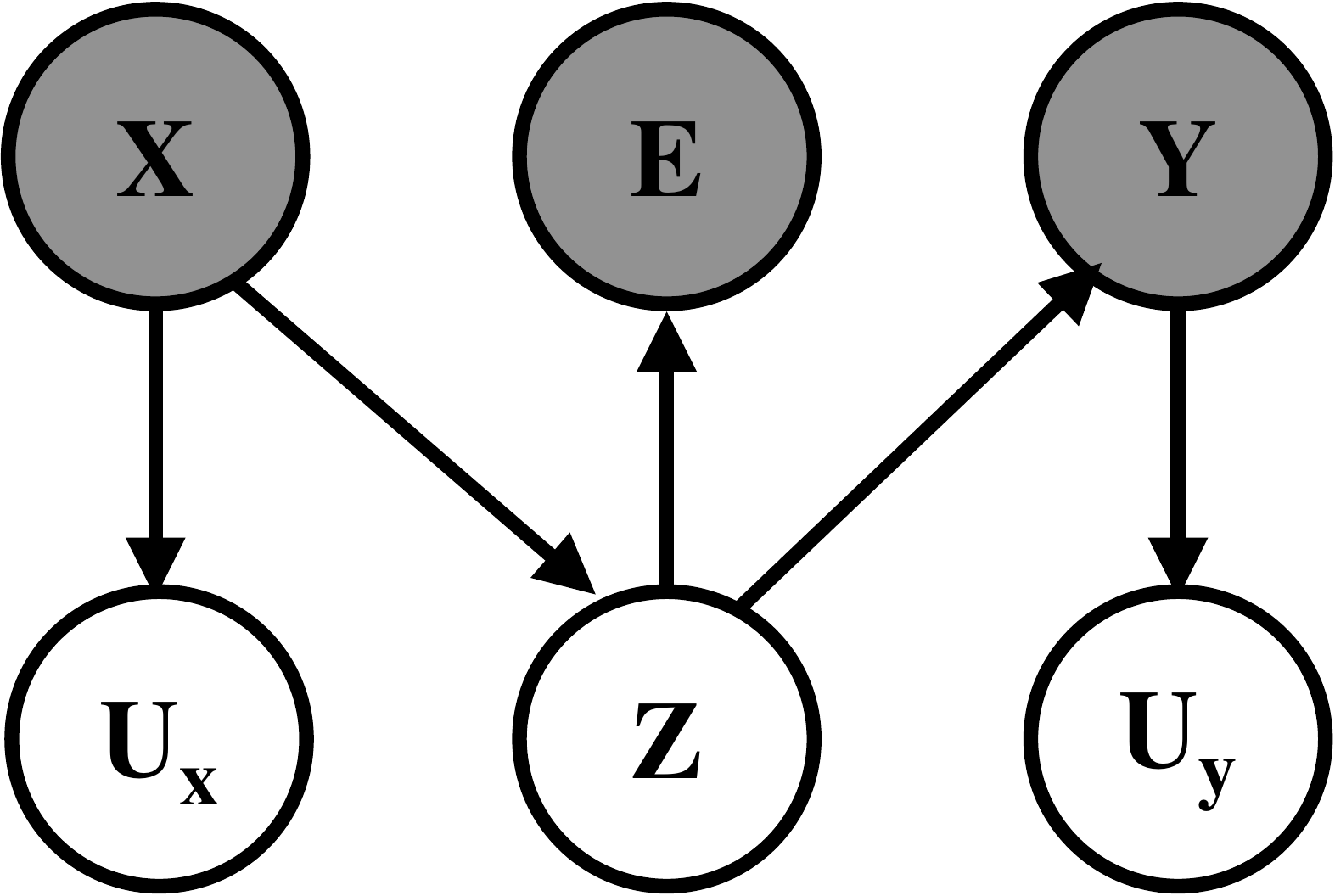} \\
$\gG^q$ & $\gG^p$ & $\gG^{\mathrm{str}}_{\mathrm{ood}}$\\
\end{tabular}
\caption{Bayesian networks for out-of-distribution generalization task. $\rE$ in the diagram represents the index of the environmental factor, not the real value of $\rE$ in the data generation process}
\label{fig:bn_irm}
\end{center}
\vskip -0.2in
\end{figure}

\textbf{Problem setup}
We show that discovering of true causation against spurious correlation through invariance can be viewed as a structured latent representation learning problem.
Consider a set of environments $\gE$ indexed by $\rE$, we have a data distribution $P^e(\rX, \rY)$ for each environment $\rE = e$.
We use $\left[ \rX, \rY, \rE \right]$ to represent the observed variables, where $\rX$ is data input, $\rY$ is label and $\rE$ is the index of the corresponding environment index.
The goal of this task is predict $Y$ from $X$ in a way that the performance of the predictor in the presence of the worst $\rE$ is optimal. 
We derive an information-theoretic objective for out-of-distribution generalization task on Colored-MNIST dataset introduced in \cite{irm}.
For more details of this experiment, please see the Appendix.~\ref{ap:exp_irm} as well as the original work ~\cite{causal_invariance,irm}.  

\textbf{Framework specification}~
As our structural regularization,  we use the Bayesian network $\gG^{\mathrm{str}}_{\mathrm{ood}}$ in Figure~\ref{fig:bn_irm}.
The purpose of $\gG^{\mathrm{str}}_{\mathrm{ood}}$ is to enforce that $Z$ is sufficient statistic in making the prediction of $Y$ and that $E \perp Y | Z$. 
%We could use the Bayesian network $\gG^{\mathrm{str}}_{\mathrm{ood}}$ in Figure~\ref{fig:bn_irm} as structural regularization to enforce that $Z$ is sufficient statistic in making the prediction of $Y$ and that $E \perp Y | Z$. 
We present the derived learning objective here
\begin{equation}
\begin{aligned}
&\Ls_{\mathrm{info}} = \ldist + \beta_1 \kl{\qq(\rvz \mid \rvx, \rve, \rvy)}{\qq(\rvz \mid \rvx)} + \\
& \beta_2 \cI_q(\rvx,\rve,\rvy \mid \rvz)
\end{aligned}
\end{equation}
We further show that the idea in~\cite{irm} can be directly integrated into our proposed framework by imposing stable $\rrM^p$ structure as constraints across environments, measured by gradient-penalty, as discussed in Appendix.~\ref{ap:irm}.

\begin{table}[t]
\caption{Out-of-distribution generalization results on Colored-MNIST}
\label{tb:irm}
\vskip 0.15in
\begin{center}
\begin{small}
\begin{sc}
\begin{tabular}{lcc}
\toprule
Model      & ACC. train envs. & Acc. test env.\\ 
\midrule
Random       & $50$ & $50$\\
Optimal      & $75$ & $75$\\
Oracle       & $73.5 \pm 0.2$ & $73.0 \pm 0.4$\\
ERM          & $87.4 \pm 0.2$ & $17.1 \pm 0.6$\\
IRM          & $70.8 \pm 0.9$ & $66.9 \pm 2.5$ \\
Ours-full    & $67.8 \pm 6.8$ & $62.1 \pm 6.1$ \\
Ours-semi    & $71.4 \pm 6.1$ & $58.7 \pm 7.2$\\
\bottomrule
\end{tabular}
\end{sc}
\end{small}
\end{center}
\vskip -0.1in
\end{table}

\textbf{Experiments}
We validate the proposed model on the Colored-MNIST classification task introduced in~\cite{irm}.
We also take the advantage of our proposed framework as a generative model that we could perform semi-supervised learning, where we use only $50\%$ labeled data.
We include more training setting details in Appendix.~\ref{ap:exp_irm}. 
We compare our model against the baselines in Table~\ref{tb:irm}.
We see that our proposed information-theoretic objective achieves comparable performance in both supervised and semi-supervised setting on the test-environment.

\section{Conclusion}
In this work, we propose a general information-theoretic framework for learning structured latent factors from multivariate data, by generalizing the multivariate information bottleneck theory.
We show that the proposed framework can provide an unified view of many existing methods and insights on new models for many different challenging tasks like fair representation learning and out-of-distribution generalization.

\bibliography{main_arxiv}
\bibliographystyle{icml2020}

\clearpage
\appendix
\onecolumn
\section{Proofs}
\subsection{Proofs of results in section ~\ref{sec:framework} \emph{framework}}
\label{ap:framework}

\subsubsection{Generalized MIB objective}
We generalized the original MIB structural variational learning objective in~\eqref{eq:loss}.
We show that by choosing $\sC_1 = \kl{\qq}{\pp},\; T = 1$ and $\gG^1 = \gG^{\emptyset}, \; K=1$, we can recover the original MIB objective~\eqref{eq:mib_loss}.
\begin{prop}
Let $\rX \sim P(\rX)$, and let $\gG^{\emptyset}$ be an empty Bayesian network over $\rX$. Then
\begin{equation}
\projkl{p}{\gG^{\emptyset}} = \min_{q \models \gG} \kl{p}{q} = \cI_{p}(\rX) - \cI^{\gG^{\emptyset}}_{p}(\rX) = \cI_{p}(\rX)
\end{equation}
\end{prop}
\begin{proof}
By definition, we have $\cI^{\gG^{\emptyset}}_{p}(\rX) = 0$. 
\end{proof}
Then we can see that our objective is equivalent to the original MIB objective~\eqref{eq:mib_loss} when $\alpha_1 = 1, \beta_1 = \gamma$.
\begin{equation}
\Ls = \ldist + \lreg = \alpha_1 \kl{\qq}{\pp} + \beta_1 \projkl{q_\rvphi}{\gG^{\emptyset}} = \alpha_1 \kl{\qq}{\pp} + \beta_1 \cI^{\gG^q}_q
\end{equation}
\subsubsection{Derivation of~\eqref{eq:inference_poe}}
\begin{equation}
\begin{aligned}
\qq(\rvz \mid \rvx^{\sS}) &\propto \pp(\rvz) \prod_{i \in \sS} \frac{\qq(\rvz \mid \rvx_i)}{\pp(\rvz)}\\
&= \pp(\rvz) \prod_{i \in \sS}\prod_{j=1}^M \left(\qqt(\rvz_j \mid \rvx_i) \right)^{\rvm^q_{ij}}\\
&= \prod_{j=1}^M \left(\pp(\rvz_j)\prod_{i \in \sS} \left(\qqt(\rvz_j \mid \rvx_i) \right)^{\rvm^q_{ij}} \right)
\end{aligned}
\end{equation}

\subsubsection{Full table~\ref{tb:unified_models}}
We show the full Table~\ref{tb:unified_models} in Table~\ref{tb:unified_models_full}.
\begin{table*}[t]
\caption{A unified view of \{single/multi\}-\{modal/domain/view\} models}
\label{tb:unified_models_full}
% \vskip 0.15in
\begin{center}
\begin{small}
\begin{sc}
\begin{tabular}{lccccccl}
\toprule
Models & $N$ & \textcircled{1} & \textcircled{2} & $\gG^q$ & $\gG^p$ &  $\ldist$ & $\lreg$ \\
\midrule
VAE    & $1$ & $\times$ & $\times$ & $\left[\gG^q_{\mathrm{single}}\right]$ & $\left[\gG^p_{\mathrm{single}}\right]$ & $[1, C_1]$ & $[]$\\
\midrule
ICA    & $1$ & $\times$ & $\times$ & $\left[\gG^q_{\mathrm{single}}\right]$ & $\left[\right]$ & $[]$ & $[\beta, \gG^p_{\mathrm{single}}]$\\
\midrule
GAN    & $1$ & $\times$ & $\times$ & [] & $\gG^p_{\mathrm{single}}$ & $[1, C_2]$ & $[]$\\
\midrule
InfoGAN & $1$ & $\times$ & $\times$ & [] & $\gG^p_{\mathrm{single}}$ & $[1, C_2]$ & $[1, \gG^{\mathrm{InfoGAN}}]$\\
\midrule
$\beta$-VAE & $1$ & $\times$ & $\times$ & $\left[\gG^q_{\mathrm{single}}\right]$ & $\left[\gG^p_{\mathrm{single}}\right]$ & $[1, C_1], [\beta-1, C_3]$ & $[\beta - 1, \gG^\emptyset]$\\
\midrule
$\beta$-TCVAE & $1$ & $\times$ & $\times$ & $\left[\gG^q_{\mathrm{single}}\right]$ & $\left[\gG^p_{\mathrm{single}}\right]$ & $[1, C_1], [\alpha_2, C_3]$ & $[\beta, \gG^p]$\\
\midrule
BiVCCA & $2$ & $\times$ & $\times$ & $\left[\gG^q_{\mathrm{marginal}}\right]$ & $\left[\gG^p_{\mathrm{joint}}\right]$ & $[\alpha_i, C_4(\rvx_i, \rvz)]]$ & $[]$\\
\midrule
JMVAE  & $2$ & $\times$ & $\times$ & $\left[\gG^q_{\mathrm{joint}}\right]$ & $\left[\gG^p_{\mathrm{joint}}\right]$ & $[1, C_1]$& $[\beta_i, \gG^{\mathrm{str}}_{\mathrm{cross}}(\rvx_i)]$\\
\midrule
TELBO  & $2$ & $\times$ & $\times$ & $\left[\gG^q_{\mathrm{joint}},\gG^q_{\mathrm{marginal}}\right]$ & $\left[\gG^p_{\mathrm{joint}}\right]$ &  $[1, C_1]$ & $[\beta_i, \gG^{\mathrm{str}}_{\mathrm{marginal}}(\rvx_i)]$\\
\midrule
MVAE   & $N$ & $\times$ & $\times$ & $\left[\gG^q_{\mathrm{joint}},\gG^q_{\mathrm{marginal}}\right]$ & $\left[\gG^p_{\mathrm{joint}}\right]$ & $[1, C_1]$ & $[\beta_i, \gG^{\mathrm{str}}_{\mathrm{marginal}}(\rvx_i)]$\\
\midrule
Wyner  & $2$ & $\checkmark$ & $\times$ & $\left[\gG^q_{\mathrm{joint}},\gG^q_{\mathrm{marginal}}\right]$ & $\left[\gG^p_{\mathrm{joint}}\right]$ & $[1, C_1]$ & $[\beta_i, \gG^{\mathrm{str}}_{\mathrm{cross}}(\rvx_i)],[\beta_i, \gG^{\mathrm{str}}_{\mathrm{private}}(\rvx_i)]$\\
\midrule
DIVA  & $3$ & $\checkmark$ & $\times$ & $\left[\gG^q_{\mathrm{marginal}}\right]$ & $\left[\gG^p_{\mathrm{joint}}\right]$ & $[1, C_1]$ & $[\beta_i, \gG^{\mathrm{str}}_{\mathrm{private}}(\rvx_i)]$ \\
\midrule
OURS-MM & $N$ & $\checkmark$ & $\checkmark$ & $\left[\gG^q_{\mathrm{full}} \right]$ & $\left[\gG^p_{\mathrm{full}} \right]$ & $[1, C_0]$ & $[\beta_i, \gG^{\mathrm{str}}_{\mathrm{cross}}(\{\rvx_i\})]$\\
% \midrule
% OURS-FR & $N$ & $\checkmark$ & $\checkmark$ & $\left[\gG^q_{\mathrm{full}} \right]$ & $\left[\gG^p_{\mathrm{full}} \right]$ & $[1, C_0]$ & $[\beta_i, \gG^{\mathrm{str}}_{\mathrm{fairness}}]$\\
% \midrule
% OURS-IRM & $N$ & $\checkmark$ & $\checkmark$ & $\left[\gG^q_{\mathrm{full}} \right]$ & $\left[\gG^p_{\mathrm{full}} \right]$ & $[1, C_0]$ & $[\beta_i, \gG^{\mathrm{str}}_{\mathrm{ood}}]$\\
\bottomrule
\end{tabular}
\end{sc}
\end{small}
\end{center}
\vskip -0.1in
\end{table*}

\subsection{Proof of results in section~\ref{sec:single_model_model} \emph{single-modal generative mode}}
\label{ap:sm}
\subsubsection{Unifying disentangled generative models}
\textbf{$\beta$-VAE}
For $\beta$-vae we have
\begin{equation}
\begin{aligned}
\Ls = &\ldist + \lreg\\
 = & C_1 + (\beta - 1) C_3 + (\beta - 1)\lreg(\gG^{\emptyset})\\
 = & C_1 + (\beta - 1) C_3 + (\beta - 1)\projkl{q_\rvphi}{\gG^{\emptyset}}\\
 = & \E_{\qq} \log \pp(\rvx \mid \rvu) + \E_{\qq}\kl{\qq(\rvu \mid \rvx)}{\pp(\rvu)} + (\beta - 1) \kl{\qq(\rvu)}{\pp(\rvu)} + (\beta - 1)\miq{\rvx}{\rvu}\\
 = & \E_{\qq} \log \pp(\rvx \mid \rvu) + (1 + \beta - 1) \kl{\qq(\rvu)}{\pp(\rvu)} + (1 + \beta - 1) \miq{\rvx}{\rvu}\\
 = & \E_{\qq} \log \pp(\rvx \mid \rvu) + \beta \E_{\qq}\kl{\qq(\rvu \mid \rvx)}{\pp(\rvu)}\\
\equiv & \Ls_{\beta\mathrm{-vae}} 
\end{aligned}
\end{equation}
where we include the structural regularization $\lreg$ using an empty Bayesian network $\gG^{\beta\mathrm{-vae}} \equiv \gG^\emptyset$.
Thus we show that the $\beta$-vae objective is equivalent to imposing another empty Bayesian network structure in the latent space which implies the independent latent factors.

\textbf{TCVAE~\cite{tcvae}}
We further show that how we can unify other total-correlation based disentangled representation learning models~\cite{tcvae,hfvae,kimdisentangle} by explicitly imposing Bayesian structure $\gG^p$ as structural regularization, where a factorized prior distribution is assumed.
\begin{equation}
\begin{aligned}
&\Ls = C_1 + \alpha_2 C_3 + \beta \lreg  \\
&\lreg = \projkl{q_\rvphi}{\gG^p} = \cI^{\gG^q}_q - \cI^{\gG^p}_q = \sum_j^{M} \miq{\rvx}{\rvu_j} - \miq{\rvx}{\rvu} = \cI_q(\rvu) - \cI_q(\rvu \mid \rvx) = \cI_q(\rvu) \equiv TC(\rvu)
\end{aligned}
\end{equation}
Since we assume a factorized posterior distribution $\qq(\rvu \mid \rvx)$, we have $\cI_q(\rvu \mid \rvx) = 0$ in the last line of above objective.
Thus the total-correlation minimization term emerges as a structural regularization term naturally in our framework.

\subsection{Proof of results in section~\ref{sec:multi_modal_model} \emph{multi-modal/domain/view generative model}}
\label{ap:mm}
\subsubsection{Unifying multi-modal/domain/view generative models}
We show that we can obtain several representative multi-modal generative models as special cases of our proposed framework here.

\textbf{JMVAE~\cite{jmvae}}
We can see that the objective of JMVAE is a speacial case of our proposed objective when $N=2$.

\textbf{Wyner-VAE~\cite{wynervae}}~
By using structural regularization $\projkl{\qq}{\gG^{\mathrm{str}}_{\mathrm{cross}}(\rvx_i)}$, we show that we can obtain the mutual information regularization term appeared in the learning objective of Wyner-VAE~\cite{wynervae}
\begin{equation}
\begin{aligned}
\lreg &= \projkl{\qq}{\gG^{\mathrm{str}}_{\mathrm{cross}}(\rvx_i)} = \cI^{\gG^q}_q - \cI^{\gG^{\mathrm{str}}_{\mathrm{cross}}(\rvx)}_q\\
&= \miq{\rvx_1}{\rvu_1} + \miq{\rvx_2}{\rvu_2} + \miq{\rvx_1, \rvx_2}{\rvz} - \miq{\rvx_1}{\rvu_1} - \miq{\rvx_2}{\rvu_2} = \miq{\rvx_1, \rvx_2}{\rvz}\\
\Ls &= \ldist + \lreg = \beta \miq{\rvx_1, \rvx_2}{\rvz} + \ldist \equiv \Ls_{\mathrm{wyner-vae}}
\end{aligned}
\end{equation}

\textbf{CorEx~\cite{corex}}
One of the most interesting model with similar goal to decorrelate observed variables is CorEx~\cite{corex,corex-hierarchical,corex-infosieve,corex-vae}, whose objective is
\begin{equation}
\begin{aligned}
\max _{G_{j}, \qq\left(\rvz_{j} \mid \rvx_{G_{j}}\right)} &\Ls_{CorEx}=\sum_{j=1}^{M} TC\left(\rvx_{G_{j}} \right) - TC\left(\rvx_{G_{j}} \mid \rvz_{j}\right)\\
&\text {s.t.} \quad G_j \cap G_{j^{\prime} \neq j}=\emptyset
\end{aligned}
\end{equation}
For each $1 \le j \le M$, CorEx objective aims to search for a latent variable $\rZ_j$ to achieve maximum total-correlation reduction $TC\left(\rvx_{G_{j}} \right) - TC\left(\rvx_{G_{j}} \mid \rvz_{j}\right)$ of a group of observed variables $\rX_{G_j}$. We use $\rrM^q_{:,j}$ and $\rrM^p_{i, :}$ to represent $G_j$ equivalently, then our objective is
\vspace{-0.4in}
\begin{equation}
\begin{aligned}
& \ldist = \projkl{\qq}{\gG^p} = \cI_q^{\gG^q} - \cI_q^{\gG^p} = \sum_{j=1}^M \miq{\rvz_j}{\rvx^{\rvm^q_j}} - \sum_{i=1}^N \miq{\rvz^{\rvm^p_i}}{\rvx_i}\\
&= \sum_{j=1}^M\sum_{i=1}^N \rvm^q_{ij}\miq{\rvz_j}{\rvx_i} + \sum_{j=1}^M\left[\cI_q(\rvx^{\rvm^q_j} \mid \rvz_j) - \cI_q(\rvx^{\rvm^q_j}) \right] - \sum_{i=1}^N\sum_{j=1}^M \rvm^p_{ij}\miq{\rvz_j}{\rvx_i} - \sum_{i=1}^N\left[\cI_q(\rvz^{\rvm^p_i} \mid \rvx_i) - \cI_q(\rvz^{\rvm^p_i}) \right]\\
&\le \sum_{j=1}^M\left[\cI_q(\rvx^{\rvm^q_j} \mid \rvz_j) - \cI_q(\rvx^{\rvm^q_j}) \right] + \sum_{i=1}^N\I_q(\rvz^{\rvm^p_i})\\
&\equiv -\Ls_{CorEx} + \sum_{i=1}^N\cI_q(\rvz^{\rvm^p_i})
\end{aligned}
\end{equation}
Thus with structural regularization $\gG^p$ we obtained an objective coincides with CorEx-based variational autoencoder~\cite{corex-vae}, which is also upper-bound of original CorEx objective\cite{corex} with additional disentanglement regularization over latent variables.

\subsubsection{Derivation of objective~\eqref{eq:jsd_objective}}
We show the detailed derivation of the learning objective of our multi-domain generative model here.
As introduced in~\ref{sec:multi_modal_model}, we impose $N$ structural regularization for each individual $\rX^{\sS} = \{ \rX_i\}$ as $\projkl{\qq}{\gG^{\mathrm{str}}_{\mathrm{cross}}\left( \left\{ \rvx_i\right\} \right)}$. First we hvae
% We established a tractable upper-bound of the original objective as a sum of $\Ls_{\rvx}$, $\Ls_{\rvu}$ and $\Ls_{\rvz}$, where $\lreg$ is upper-bounded by $\Ls_{\rvz}$.
\begin{prop}
We have following upper-bound
\begin{equation}
\frac{1}{N} \sum_{i=1}^N \projkl{\qq}{\gG^{\mathrm{str}}_{\mathrm{cross}}\left( \left\{ \rvx_i\right\} \right)} \le \Ls_{\rvu} + \sum_{i=1}^N \E_{\qq} \kl{\qq(\rvz \mid \rvx)}{\qq(\rvz \mid \rvx_i)}
\end{equation}
\end{prop}
\begin{proof}
\begin{align*}
&\frac{1}{N} \sum_{i=1}^N \projkl{\qq}{\gG^{\mathrm{str}}_{\mathrm{cross}}\left( \left\{ \rvx_i\right\} \right)} = \miq{\rvu}{\rvx} + \frac{1}{N} \sum_{i=1}^N \left[ \miq{\rvz}{\rvx} -  \miq{\rvz}{\rvx_i} - \sum_{k \neq i}^N \miq{\rvz}{\rvx_k} \right]\\
&= \miq{\rvu}{\rvx} + \frac{1}{N}\sum_{i=1}^N \miq{\rvz}{\rvx} + \frac{1}{N}\sum_{i=1}^N \left[ - \miq{\rvz}{\rvx_i} + \sum_{k \neq i}^N \miq{\rvz}{\rvx_k} \right]\\
&= \miq{\rvu}{\rvx} +  \miq{\rvz}{\rvx} -  \sum_{i=1}^N \miq{\rvz}{\rvx_i}\\
% &= \E_{\qq}\log \frac{\qq(\rvz \mid \rvx)}{\qq(\rvz)} - \log \frac{\qq(\rvz \mid \rvx_i)}{\qq(\rvz)}\\
% &= \E_{\qq} \log \frac{\qq(\rvz \mid \rvx)}{\qq(\rvz \mid \rvx_i)}\\
% &\le \E_{\qq} \log \frac{\prod_{j=1}^M \prod_{i=1}^N \qq(\rvz_j \mid \rvx)^{\rvm^q_{ij}}}{\prod_{j=1}^M \prod_{i=1}^N \qq(\rvz_j \mid \rvx_i)^{\rvm^q_{ij}}}\\
&= \E_{\qq}  \kl{\qq(\rvu \mid \rvx)}{\qq^{\mathrm{mg}}(\rvu)} + \sum_{i=1}^N \E_{\qq} \kl{\qq(\rvz \mid \rvx)}{\qq^{\mathrm{mg}}(\rvz \mid \rvx_i)} \\
&= \E_{\qq}  \kl{\qq(\rvu \mid \rvx)}{\pp(\rvu)} + \sum_{i=1}^N \E_{\qq} \kl{\qq(\rvz \mid \rvx)}{\qq(\rvz \mid \rvx_i)} \\
&- \E_{\qq}  \kl{\qq^{\mathrm{mg}}(\rvu)}{\pp(\rvu)} - \sum_{i=1}^N \E_{\qq} \kl{\qq^{\mathrm{mg}}(\rvz \mid \rvx)}{\qq(\rvz \mid \rvx_i)} \\
&\le \E_{\qq}  \kl{\qq(\rvu \mid \rvx)}{\pp(\rvu)} + \sum_{i=1}^N \E_{\qq} \kl{\qq(\rvz \mid \rvx)}{\qq(\rvz \mid \rvx_i)}\\
&= \Ls_{\rvu} + \sum_{i=1}^N \E_{\qq} \kl{\qq(\rvz \mid \rvx)}{\qq(\rvz \mid \rvx_i)}
\end{align*}
\end{proof}
where $\qq^{\mathrm{mg}}(\rvu) \equiv \E_{\qq} \qq(\rvu \mid \rvx) $ and $\qq^{\mathrm{mg}}(\rvz \mid \rvx_i) = \E_{\qq(\rvx \mid \rvx_i )} \qq(\rvz \mid \rvx)$ denote the induced marginalization of $\qq(\rvx, \rvu, \rvz)$.
Note that by using the above upper-bound, the inference network distribution $\qq(\rvz \mid \rvx_i)$ introduced in~\ref{sec:framework_inference} 
is trained to approximate the true marginalization $\qq^{\mathrm{mg}}(\rvz \mid \rvx)$.
Thus we have following full objective
\begin{align}
\Ls &= \ldist + \lreg = \kl{\qq(\rvx, \rvz, \rvu)}{\pp(\rvx, \rvz, \rvu)} + \frac{1}{N} \sum_{i=1}^N \projkl{\qq}{\gG^{\mathrm{str}}_{\mathrm{cross}}\left( \left\{ \rvx_i\right\} \right)} \nonumber\\
    % & \le -\E_{\qq(\rvz, \rvu \mid \rvx)}\log \pp(\rvx \mid \rvz, \rvu) + \E_{\qq(\rvx)}\kl{\qq(\rvu \mid \rvx)}{\pp(\rvu)} + \E_{\qq(\rvx)}\kl{\qq(\rvz \mid \rvx)}{\pp(\rvz)} + \frac{1}{N} \sum_{i=1}^N \projkl{\qq}{\gG^{\mathrm{str}}_{\mathrm{cross}}\left( \left\{ \rvx_i\right\} \right)}\\
    &= -\E_{\qq(\rvz, \rvu \mid \rvx)}\log \pp(\rvx \mid \rvz, \rvu) \tag{$\Ls_{\rvx}$}\\
    & + \E_{\qq(\rvx)}\kl{\qq(\rvu \mid \rvx)}{\pp(\rvu)} \tag{$\Ls_{\rvu}$} \\
    & + \sum_{i=0}^N \E_{\qq(\rvx)}\kl{\qq(\rvz \mid \rvx)}{\qq(\rvz \mid \rvx_i)} \tag{$\Ls_{\rvz}$} \\
    &\equiv \Ls_{\rvx} + \Ls_{\rvu} + \Ls_{\rvz}
\end{align}
We use $\qq(\rvz \mid \rvx_0) \equiv \pp(\rvz)$ for the simplicity of notations.
We further show that $\Ls_{\rvz}$ can be viewed as a generalized JS-divergence for the reverse KL-divergence~\cite{jsd_abs_mean}.
We decompose $\Ls_{\rvz}$ regarding each latent variable $\rZ_j$, 
\begin{equation}
\begin{aligned}
&\Ls_{\rvz} = \sum_{j=1}^M \Ls_{\rvz_j}, \quad \qq(\rvz_j \mid \rvx) \propto \prod_{i=0}^N \qq(\rvz_j \mid \rvx_i)^{\gamma_{ij}}\\
&\Ls_{\rvz_j} = D^{\mathrm{KL}^\ast}_{\mathrm{JS}}\left(\qq(\rvz_j \mid \rvx_0), \qq(\rvz_j \mid \rvx_1), \ldots, \qq(\rvz_j \mid \rvx_N) \right)\\
% &D^{\mathrm{KL}^\ast}_{\mathrm{JS}}\left(q_0, \ldots, q_N \right) = \sum_{i=0}^N \gamma_i \kl{\qq(\rvz \mid \rvx)}{\qq(\rvz \mid \rvx_i)}\\
&\sum_{i=0}^{N} \gamma_i = 1, \gamma_{0j} = 1 - \sum_{i=1}^N \rvm^q_{ij}, \quad \gamma_{ij} = \rvm^q_{ij} \; i > 0
\end{aligned}
\end{equation}
where we use $\mathrm{KL}^\ast$ to denote the reverse KLD and following the same notation in~\cite{jsd_abs_mean} for the generalized JSD.

\subsection{Proof of results in section~\ref{sec:fairness} \emph{case study: fair representation learning}}
\label{ap:fairness}
We show the detailed derivation of the learning objective~\ref{eq:fair_objective} here.
\begin{equation}
\begin{aligned}
&\Ls = \ldist + \lreg = \kl{\qq(\rvx, \rvz, \rvu)}{\pp(\rvx, \rvz, \rvu)} + \beta_1\projkl{\qq}{\gG^{\mathrm{str}}_{\mathrm{informative}}} + \beta_2\projkl{\qq}{\gG^{\mathrm{str}}_{\mathrm{invariant}}}\\
&= \kl{\qq}{\pp}  + \beta_1 \I_q(\rvx;\rva \mid \rvz) + \beta_2 \miq{\rvz}{\rvu} + const\\
&\le -\E_{\qq}\log \pp(\rvx, \rva \mid \rvz, \rvu) + \beta_2 \miq{\rvz}{\rvu} + (1+\beta_1) \E_{\qq} \kl{\qq(\rvz \mid \rvx, \rva)}{\pp(\rvz)} + const
%\E_{\qq(\rvx,\rva)} \kl{\qq(\rvz \mid \rvx, \rva)}{}
\end{aligned}
\end{equation}

 We can interpret this derived learning objective as first seeking for a succinct latent representation $\rZ$ that captures the sufficient correlation between $\rX$ and $\rA$, then $\rZ$ is served as a proxy variable to learn an informative representation $\rU$ with all information relevant to $\rA$ eliminated by minimizing $\miq{\rvz}{\rvu}$.
 
\subsection{Details of section~\ref{sec:irm} \emph{case study: invariant risk minimization}}
\label{ap:irm}
We show that the idea in~\cite{irm} can be directly integrated into our proposed framework by imposing stable $\rrM^p$ structure as constraints across environments, measured by gradient-penalty term shown below
\begin{equation}
\begin{aligned}
\Ls_{\mathrm{gp}} &= \ldist + \E_{\qq(\rve)} \mynorm{\nabla_{\rrM^p} \Ls_{\mathrm{score}}}
\end{aligned}
\end{equation}

\section{Experiments}
%Source code for all experiments in this works is available in the supplemental material, where all details about hyper-parameters and training setup can be found.
\subsection{Generative modeling}
\label{ap:exp_mm}
\textbf{Datasets}
Following the same evaluation protocol proposed by previous works~\cite{wynervae,mvae}, we construct the bi-modal datasets MNIST-Label by using the digit label as a second modality, MNIST-SVHN by pairing each image sample in MNIST with another random SVHN image sharing the same digit label and a bi-view dataset MNIST-MNIST-Plus-1 by pairing each MNIST sample $\rX_1$ with another random sample $\rX_2$ correlated as $\mathrm{label}(\rX_1) + 1 = \mathrm{label}(\rX_2)$. We illustrate the data generating process using Bayesian networks in Figure~\ref{fig:bn_mm_gt}.

\begin{figure}[ht]
\vskip 0.2in
\begin{center}
\begin{tabular}{cc}
\includegraphics[width=0.25\columnwidth]{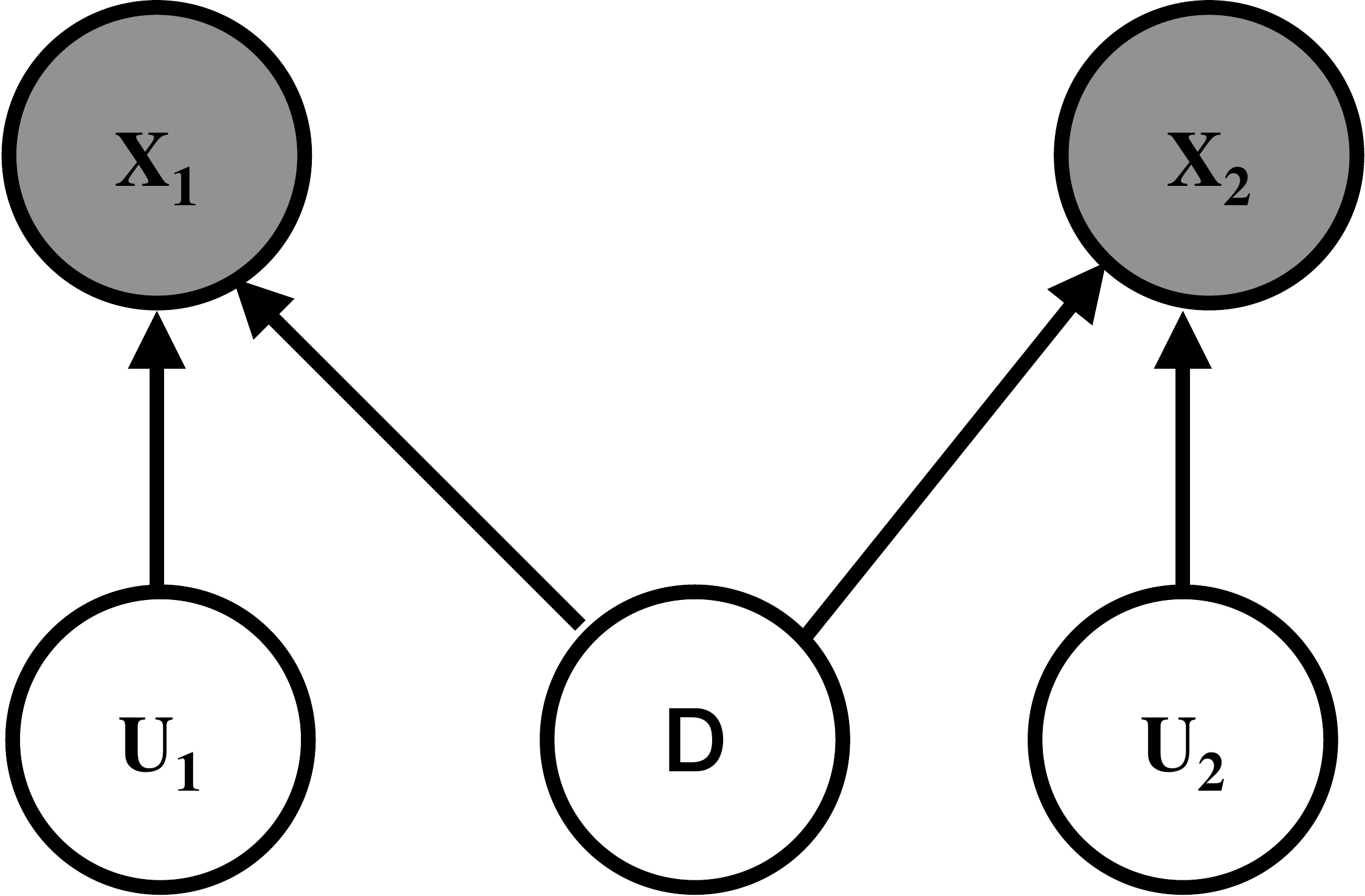} & 
\includegraphics[width=0.33\columnwidth]{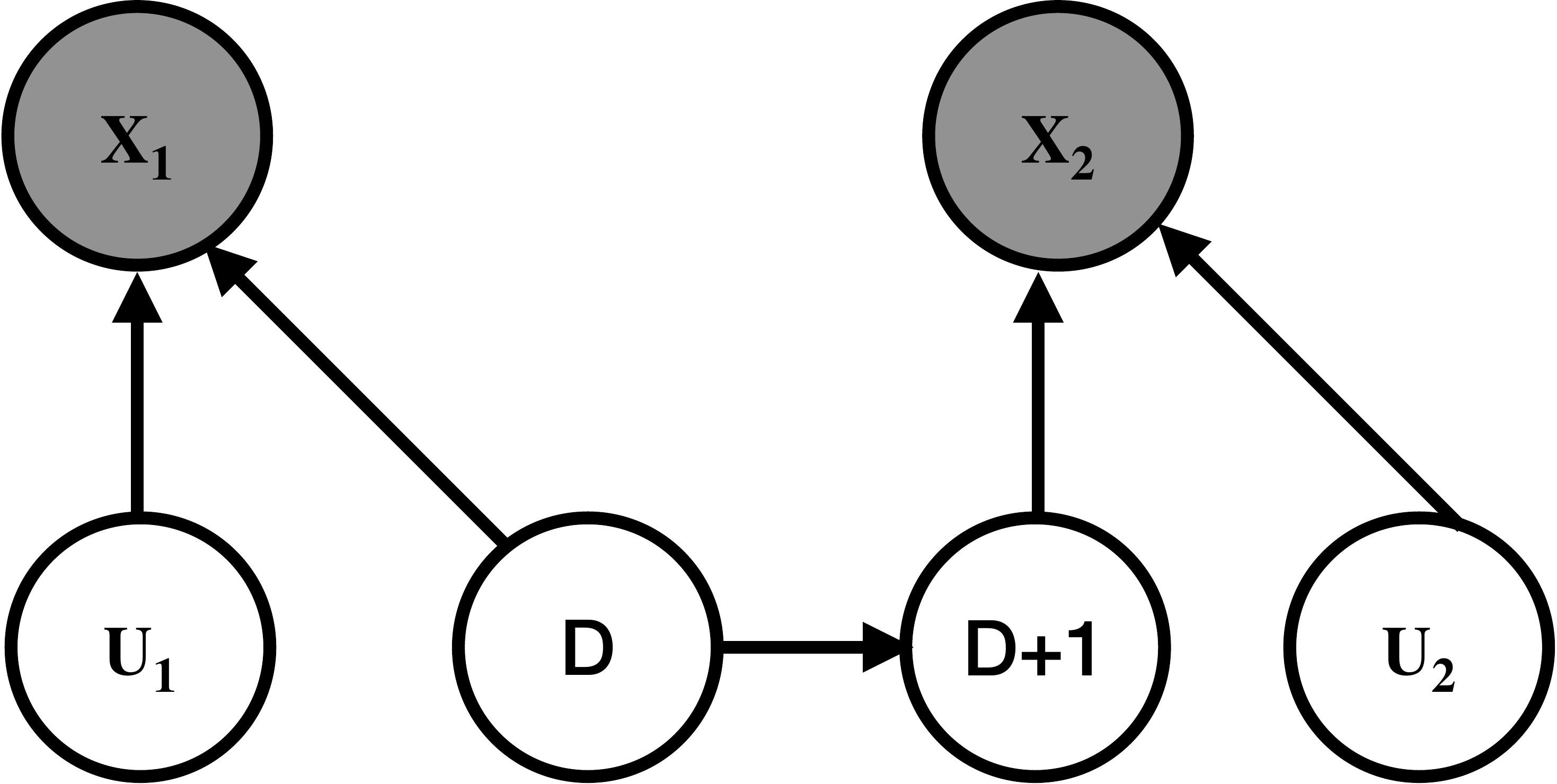}\\
MNIST-SVHN & MNIST-Plus-1\\
\end{tabular}
\caption{Bayesian networks for illustrating the data generating process of MNIST-SVHN dataset and MNIST-PLUS-1 dataset.}
\label{fig:bn_mm_gt}
\end{center}
\vskip -0.2in
\end{figure}

\textbf{Training details and hyper-parameters} For MNIST-Label dataset, we use MLPs with $2$ hidden layers for both encoders and decoders, following the same neural network architecture in~\cite{mvae}. The dimension of $\rZ$ modeling the shared information is $2$. The dimension of $\rU_1$ modeling MNIST image is $20$. We don't include $\rU_2$ in this setting and set the dimension of $\rU_2$ to $0$. For MNIST-SVHN dataset, the dimension of $\rZ$ is $2$, the dimension of $\rU_1$ for MNIST is $20$ and the dimension of $\rU_2$ for SVHN is $20$. For MNIST-MNIST-Plus-1 dataset, the dimension of $\rZ$ is $2$, and the dimension of $\rU_1$ for MNIST is $20$. We train the model using the Adam optimizer with a learning rate starting from $0.001$, and decay the learning rate by a factor $0.1$ whenever a validation loss plateau is found during training. We train the model up to $1000$ epochs for all datasets. We learn the structural variable $\rrM$ with $steps\_dist=1$ and $steps\_str=3$ in all experiments. We use the same neural network architectures for encoder and decoder as~\cite{wynervae} in MNIST-SVHN and MNIST-MNIST-Plus-1 datasets.

\begin{figure}[ht]
\vskip 0.2in
\begin{center}
\begin{tabular}{cc}
\includegraphics[width=0.4\columnwidth]{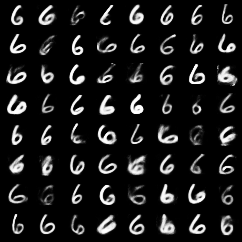} & 
\includegraphics[width=0.4\columnwidth]{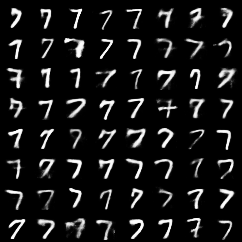}\\
$a$ & $b$
\end{tabular}
\caption{Conditionally generated samples when (a) $label = 6$ and (b) $label= 7$. }
\label{fig:mnist_label}
\end{center}
\vskip -0.2in
\end{figure}

\textbf{Qualitative results of MNIST-Label} Due to the space limit constraint, we include the qualitative results of MNIST-Label experiment here. We show the conditionally generated samples in figure~\ref{fig:mnist_label}.

\subsection{Fairness}
\label{ap:exp_fair}
\textbf{Training details and hyperparameter sensitivity} We follow the same neural network architecture design and evaluation process in~\cite{fair_lagvae}. The dimension of $\rU$ is $10$ for German and Adult datasets, the dimension of $\rZ$ is $5$. We find that the experimental result is not sensitive to the dimension of $\rZ$ when it's in range $2$ to $10$. We train the model up to $10000$ epochs using Adam optimizer with leraning rate $0.001$, and decay the learning rate by a factor $0.1$ when loss plateau is detected. We don't train the structural variables in this experiment. We re-scale the likelihood in objective to make the loss terms balance for the consideration of training stability. Numbers in table~\ref{tb:fair} are evaluated with $10$ random runs with different random seeds.

\subsection{Out-of-Distribution Generalization}
\label{ap:exp_irm}

\begin{figure}[th]
\vskip 0.1in
\begin{center}
\begin{tabular}{cc}
\includegraphics[width=0.45\columnwidth]{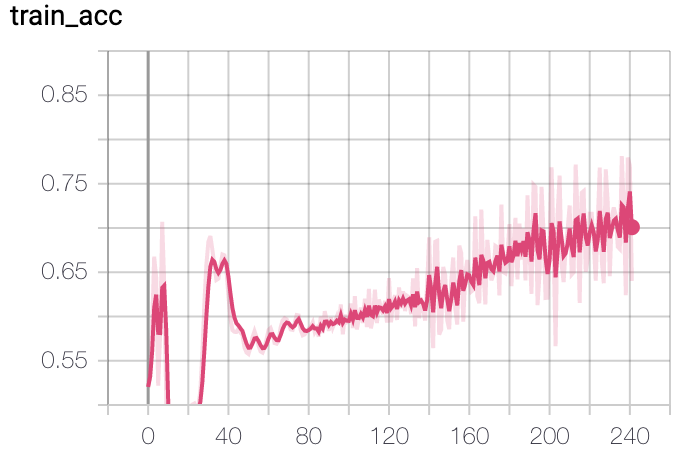} & 
\includegraphics[width=0.45\columnwidth]{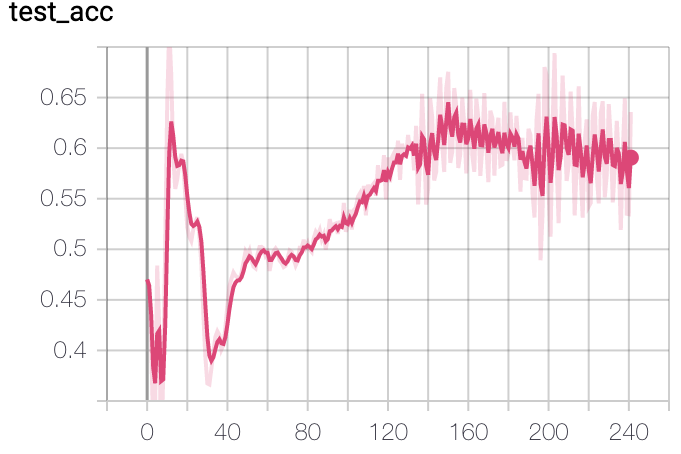}\\
% \includegraphics[width=0.3\columnwidth]{Project/icml/figures/fair_4-cropped.pdf} \\
% $\gG^{\mathrm{str}}_{\mathrm{informative}}$ & $\gG^{\mathrm{str}}_{\mathrm{invariant}}$\\
\end{tabular}
\caption{Training environment accuracy (Left) and testing environment accuracy (Right) on Colored-MNIST dataset}
\label{fig:irm_acc}
\end{center}
\vskip -0.2in
\end{figure}

\textbf{Colored MNIST} \textit{Colored MNIST} is an experiment that was used in \citep{irm}, in which the goal is to predict the label of a given digit in the presence of varying exterior factor $e$.
The dataset for this experiment is derived from MNIST. 
Each member of the \textit{Colored MNIST} dataset is constructed from an image-label pair $(x, y)$ in  MNIST, as follows.
\vspace{-0.7\baselineskip} 
{\setlength{\leftmargini}{15pt}
\begin{enumerate}
\setlength{\parskip}{0.025cm}
\setlength{\itemsep}{0.025cm}
    \item Generate a binary label $\hat y_{obs}$ from $y$ with the following rule: $\hat y_{obs}=0$ if $y \in \{0 \sim 4\}$ and  $\hat y_{obs}=1$ otherwise.  
    \item Produce $y_{obs}$ by flipping $\hat y_{obs}$ with a fixed probability $p$.
    \item Let $x_{fig}$ be the binary image corresponding to $y$.
    \item Put $y_{obs}= \hat x_{ch1}$, and construct $x_{ch1}$ from  $\hat x_{ch}$ by flipping $\hat x_{ch1}$ with probability $p_e$.
    \item Construct $x_{obs} = x_{fig} \times  [x_{ch0}, (1-x_{ch0}), 0]$.(that is, make the image red if $x_{ch1}=1$ and green if   $x_{ch1}=0$.) Indeed, $x_{obs}$ has exactly same information as the pair $(x_{fig},  x_{ch1})$. 
\end{enumerate}}
\vspace{-0.7\baselineskip}
The goal of this experiment is to use the dataset with $p_e$ values in small compact range (training dataset) to train a model that can perform well on all ranges of $p_e$.  
In particular, we use the dataset with $p_e \in \{0.1, 0.2\}$
and evaluate the model on the dataset with $p_e= 0.9$. 
%predict the label of the digit irrespective of the environmental factor, which is generally unknown to the users.
%The evaluation of model in this experiment is done using the dataset constructed from the set of environmental factors that is different from the one used in the construction of the training datasets. 
For more details of Colored MNIST experiment, please consult the original article. 

\textbf{Training details}~We follow the same neural network architecture design of encoder and evaluation process in~\cite{irm}. The decoder is $1$-layer MLP. We re-scale the likelihood terms to make the gradient norm of each one stays in the same magnitude. We train the model in a full-batch training manner, that the batch size is $50000$. For semi-supervised training, we randomly partitioned the dataset into two halfs and alternate between training $(\rX, \rE, \rY)$ and $(\rX, \rE)$. The dimension of $\rZ$ is $4$. Following the same practice in~\cite{irm}, we use early-stopping on validation set as regularization. Numbers in table~\ref{tb:irm} are evaluated with $10$ random runs with different random seeds. We illustrate the training dynamics of our model by plotting the accuracy progression in both training environments and testing environment in Figure~\ref{fig:irm_acc}.

%%%%%%%%%%%%%%%%%%%%%%%%%%%%%%%%%%%%%%%%%%%%%%%%%%%%%%%%%%%%%%%%%%%%%%%%%%%%%%%
%%%%%%%%%%%%%%%%%%%%%%%%%%%%%%%%%%%%%%%%%%%%%%%%%%%%%%%%%%%%%%%%%%%%%%%%%%%%%%%

\end{document}